\documentclass[11pt]{article}
\usepackage[utf8]{inputenc}

\usepackage{color}
\usepackage{latexsym,dsfont}
\usepackage{amssymb, stmaryrd}
\usepackage{amsmath}
\usepackage{amsthm}
\usepackage{enumerate}
\usepackage{hyperref}
\usepackage{graphicx,float}
\usepackage{caption,color}
\usepackage[algo2e,ruled,vlined]{algorithm2e} 
 \usepackage{xcolor}
\definecolor{algoColorKeyword}{named}{blue}
\definecolor{algoColorComment}{named}{olive}

\SetKwComment{Comment}{$\triangleright$\ }{}
 
\newtheorem{Theorem}{Theorem}[part]
\newtheorem{Definition}{Definition}[part]
\newtheorem{Proposition}{Proposition}[part]

\newtheorem{Lemma}{Lemma}[part]

\newtheorem{Corollary}{Corollary}[part]
\newtheorem{Remark}{Remark}[part]

\def \N{\mathbb{N}}
\def \R{\mathbb{R}}

\def \E{\mathbb{E}}
\def \F{\mathbb{F}}

\def \P{\mathbb{P}}

\def \Ac{{\cal A}}

\def \Cc{{\cal C}}
\def \Dc{{\cal D}}

\def \Fc{{\cal F}}

\def \Lc{{\cal L}}

\def \Yc{{\cal Y}}

\def \Zc{{\cal Z}}

\def \Rc{{\cal R}}

\def \1{{\mathds 1}}

\def \eps{\varepsilon}

\def \ni{\noindent}

\def\reff#1{{\rm(\ref{#1})}}

\newcommand{\nc}{\newcommand}
\nc{\esssup}{\mathop{\mathrm{ess\,sup}}}
\nc{\essinf}{\mathop{\mathrm{ess\,inf}}}

\def\beqs{\begin{eqnarray*}}
\def\enqs{\end{eqnarray*}}
\def\beq{\begin{eqnarray}}
\def\enq{\end{eqnarray}}

 \title{Discretization and Machine Learning  Approximation  \\
of BSDEs with a Constraint on the Gains-Process}
 
 \author{Idris Kharroubi\footnote{Sorbonne Universit\'e, Universit\'e de Paris, CNRS, Laboratoire de Probabilit\'es, Statistiques et Mod\'elisations (LPSM), Paris, France,  \texttt{idris.kharroubi  at upmc.fr}.} \and Thomas Lim\footnote{ENSIIE, Laboratoire de Math\'ematiques et Mod\'elisation d'Evry,
 CNRS UMR 8071, \texttt{lim at ensiie.fr}.}\and Xavier Warin \footnote{EDF R$\&$D $\&$ FiME \texttt{xavier.warin at edf.fr}}}  
\begin{document}

\maketitle
\begin{abstract} We study the approximation of backward stochastic differential equations (BSDEs for short) with a constraint on the gains process. We first discretize the constraint by applying a so-called facelift operator at times of a grid. We show that this discretely constrained BSDE converges to the continuously constrained one as the mesh grid converges to zero.   We then focus on the approximation of the discretely constrained BSDE. For that we adopt a machine learning approach. We show that the facelift can be approximated by an optimization problem over a class of neural networks under constraints on the neural network and its derivative. We then derive an algorithm converging to the discretely constrained BSDE 
 as the number of neurons goes to infinity. We end by numerical experiments. 
\end{abstract}




\section{Introduction}
In this paper, we propose an algorithm for the numerical resolution of  BSDEs with a constraint on the gains process.
Namely, we consider the approximation of the minimal solution to the BSDE
\beqs
Y_t & = & g(X_T)+\int_t^Tf(X_s,Y_s,Z_s)ds-\int_t^TZ_s.dB_s+K_T-K_t\;,\quad t\leq T
\enqs 
with constraint
\beqs
Z & \in & \sigma^{\top}(X) \Cc \;,\quad  dt\otimes d\P-a.e.
\enqs
Here,  $\Cc$ is a closed convex set, $K$ is a nondecreasing process, $B$ is a $d$-dimensional Brownian motion and $X$ solves the SDE
\beqs
dX_t & = & b(X_t)dt+\sigma(X_t)dB_t\;.
\enqs
This kind of equation is related to the super-replication under portfolio constraints in mathematical finance (see  e.g. \cite{EKPQ97}). 
A first approach to show existence of minimal solutions was done in \cite{CKS98} using a duality approach. As far as we know the most general result is given in \cite{Peng99} where the existence of a minimal solution is a byproduct of a general limit theorem for supersolutions of Lipschitz BSDE. In particular, the minimal solution is characterized as the limit of penalized BSDEs. 

 As far as we know, this characterization of the constrained solution as limit of penalized BSDEs is the wider one. In particular, we cannot express in a simple way how the constraint on the $Z$ component acts on the process $Y$. 
Therefore, the construction of numerical scheme remains a challenging issue.  A possible approach can be to use the penalized BSDEs to approximate the constrained solution. However, this leads to approximate BSDEs with exploding Lipschitz constant for the generator which  gives a very slow and sometimes unstable converging scheme \cite{gobet2008numerical}. Therefore, one needs to focus on the structure of the constrained solution to set a stable numerical scheme.  

Recently,  \cite{BEM18} gives more insights on the minimal solutions of constrained BSDEs. 
 The minimal solution is proved to satisfy a classical $L^2$-type regularity -as for BSDEs without constraint- but only until $T^-$. At the terminal time $T$, the constraint leads to a boundary effect which consists in replacing the terminal value $g$ by a functional transformation $F_\Cc[g]$ called facelift. This facelift transformation can be interpreted as the smallest function dominating the original function such that its derivative satisfies the constraint.

Taking advantage of those recent advances, we derive a converging approximation algorithm for constrained BSDEs. 

To this end we proceed in two steps. We first provide a discrete time approximation of the constraint.
 Taking into account the boundary effect mentioned in \cite{BEM18}, we apply the facelift operator to the Markov function relating $Y$ to the underlying diffusion $X$, at the points of a given discrete grid. This leads to a new BSDE with a discrete-time constraint.
Using the regularity property provided by \cite{BEM18}, we prove a convergence result as the mesh of the constraint grid goes to zero. Let us mention the article \cite{CEK20} where a similar discretization is obtained for the super-replication price. However the approach used in  \cite{CEK20} is different and consists in the approximation of the dual formulation by restricting it to stepwise processes. 
  
  We then provide a computable algorithm to approximate the BSDE with discrete-time constraint. The main issue here comes from the facelift transformation as it involves all the values of the Markov function linking $Y$ to the underlying diffusion $X$. In particular, we cannot proceed as in the reflected case where the transformation on $Y$ depends only on its value.

To overcome this issue we adopt a machine learning approach. More precisely, we compute the facelift by neural network approximators.
Using the interpretation of the facelift as the smallest dominating function whose derivatives belong to the constraint set $\Cc$, we propose an approximation as a neural network minimizing the square error under the constraint of having derivatives in $\Cc$ and dominating the original function. We notice that this approximation turns the problem into a parametric one, which is numerically valuable.  

Using the universal approximation property of neural networks up to order one,  we show that this approximation converges to the facelift as the number of neurons goes to infinity. 
Combining our machine learning approximation of the facelift with recent machine learning approximations for BSDEs/PDEs described in \cite{HPW19}, we are able to derive a fully computable algorithm for the approximation of BSDEs with constraints on the gain process.

The remainder of paper is organized as follows. In Section \ref{sec2}, we recall the main assumptions, definitions and results on BSDEs with constraints on the gains process. In Section \ref{sec3}, we introduce the discretely constraints and prove the convergence to the continuously constrained BSDEs as the mesh of the discrete constraint grid goes to zero. In Section \ref{sec4}, we present the neural network approximation of the facelift and propose a converging approximation scheme for discretely constrained BSDEs.   \\
Finally, Section \ref{sec5} is devoted to numerical experiments. At first, we show that the numerical approximation of the facelift by a neural network is not obvious using a simple minimization with penalization of the constraints. This simple approach numerically gives an upper bound of the facelift. We then derive an original iterative algorithm that we show on examples to converge to the facelift  till dimension 10.\\
At last the whole algorithm including the facelift approximation and the BSDE resolution using the methodology in \cite{HPW19} is tested on some option pricing problems with differential interest rates.

\section{BSDEs with a convex constraint on the gains-process}\label{sec2}
\subsection{The constrained BSDE}
Given a finite time horizon $T>0$ and a finite dimension $d\geq 1$, we denote by $\Omega$ the space $C([0,T],\R^d)$ of continuous functions from $[0,T]$ to $ \R^d$. We endow this space with the Wiener measure $\P$. We denote by $B$ the coordinate process defined on $\Omega$ by $B_t(\omega)=\omega(t)$ for $\omega\in \Omega$. We then define on $\Omega$ the filtration $(\Fc_t)_{t\in[0,T]}$ defined as the $\P$-completion of the filtration generated by $B$. 

We are given two mesurable functions $b,\sigma:~[0,T]\times \R^d\rightarrow\R^d,~\R^{d\times d}$ on which we make the following assumption.

\vspace{2mm}

\textbf{(H$b,\sigma$)}\begin{enumerate}[(i)]
\item The values of the function $\sigma$ are invertible.
\item The functions $b$, $\sigma$ and $\sigma^{-1}$ are bounded: there exists a constant $M_{b,\sigma}$ such that
\beqs
|b(t,x)|+|\sigma(t,x)|+|\sigma^{-1}(t,x)| & \leq & M_{b,\sigma}
\enqs
for all $t\in[0,T]$ and $x\in\R^d$.
\item The functions $b$ and $\sigma$ are Lipschitz continuous in their space variable uniformly in their time variable: there exists a constant $L_{b,\sigma}$ such that
\beqs
|b(t,x)-b(t,x')|+|\sigma(t,x)-\sigma(t,x')| & \leq & L_{b,\sigma}|x-x'|
\enqs
for all $t\in[0,T]$ and $x,x'\in\R^d$.
\end{enumerate}

\vspace{2mm}

Under Assumption \textbf{(H$b,\sigma$)}, we can define the process 
$X^{t,x}$ as the solution to the SDE
\beqs
X^{t,x}_s & = & x+\int_t^sb(r,X^{t,x}_r)dr+\int_t^s\sigma(r,X^{t,x}_r)dB_r\;,\quad s\in[t,T],
\enqs
and by classical estimates, there exists a constant $C$ such that
\beqs
\E\Big[\sup_{s\in[t,T]}|X^{t,x}_s|^2\Big] & \leq & C 
\enqs
for all $(t,x)\in[0,T]\times\R^d$ and
\beq\label{estim-reg-diff}
\E\Big[\sup_{s\in[t\vee t',T]}|X^{t,x}_s-X^{t',x'}_s|^2\Big] & \leq & C\big(|t-t'|+|x-x'|^2\big) 
\enq
for all $t,t'\in[0,T]$ and $x,x' \in \R^d$.\\

We now define the backward equation. To this end, we consider two functions 
$f:~[0,T]\times\R^d\times\R\times\R^d\rightarrow\R$ and $g:~\R^d\rightarrow\R$ on which we make the following assumption.  

\vspace{2mm}

 \textbf{(H$f,g$)}\begin{enumerate}[(i)]
\item The function $g$ is bounded: there exists a constant $M_g$ such that
\beqs
|g(x)| \leq M_g
\enqs
for all $x\in\R^d$.
\item The function $f$  is continuous and satisfies the following growth property: there exists a constant $M_f$ such that
\beqs
|f(t,x,y,z))| & \leq & M_f\big(1+|y|+|z|\big)
\enqs
for all $t\in[0,T]$, $x\in\R^d$, $y\in\R$ and $z\in\R^d$.
\item The functions $f$ and $g$ are Lipschitz continuous in their space variables uniformly in their time variable: there exists two constants $L_f$ and $L_g$ such that
\beqs
|f(t,x,y,z)-f(t,x',y',z')| & \leq &
 L_f\big(|x-x'|+|y-y'|+|z-z'|\big) \\
 |g(x)-g(x')| & \leq &
 L_g|x-x'|\enqs
for all $t\in[0,T]$, $x,x'\in\R^d$,  $y,y'\in\R$ and $z,z'\in\R^d$.
\end{enumerate}

\vspace{2mm}

We then fix a bounded convex subset $\Cc$ of  $\R^d$ such that $0\in \Cc$. For $t\in[0,T]$, we denote by $\F^t=(\Fc^t_s)_{s\in[t,T]}$ the completion of the filtration generated by $(B_s-B_t)_{s\in[t,T]}$.
We define $
\mathbf{S}^2_{[t,T]}$ (resp. $\mathbf{H}^2_{[t,T]}$) as the set of $\R$-valued \textit{c\`adl\`ag} $\F^t$-adapted (resp. $\R^d$-valued $\F^t$-predictable) processes $U$ (resp. $V$) such that ${\|U\|}_{\mathbf{S}^2_{[t,T]}}:=\E[\sup_{[t,T]}|U_s|^2]<+\infty$ (resp.  ${\|V\|}_{\mathbf{H}^2_{[t,T]}}:=\E[\int_{t}^T|V_s|^2]ds<+\infty$). We also define $\mathbf{A}^2_{[t,T]}$ as the set of $\R$-valued nondecreasing \textit{c\`adl\`ag} $\F^t$-adapted processes $K$ such that $K_t=0$ and $\E[|K_T|^2]<+\infty$. 

A solution to the constrained BSDE with parameters $(t,x,f,g, \Cc)$ is defined as  a triplet of processes $(U,V,A)\in \mathbf{S}^2_{[t,T]}\times\mathbf{H}^2_{[t,T]}\times \mathbf{A}^2_{[t,T]}$ such that
\beq\label{EDSRC1}
U_s & = & g(X^{t,x}_T)+\int_s^Tf(u,X^{t,x}_u,U_u,V_u)du-\int_s^T V_ud B_u+A_T-A_s\quad \\
V_s & \in & \sigma(s, X_s^{t,x})^\top \Cc \label{EDSRC2}
\enq
for $s\in[t,T]$\;. 

 Under Assumptions \textbf{(H$b,\sigma$)} and \textbf{(H$f,g$)} and since $0\in \Cc$, there exists a solution to \reff{EDSRC1}-\reff{EDSRC2} given by
 \beq\label{defsursolU}
 U_s  ~ = ~& (M_g+1)e^{M_f(T-s)}-1\;,\quad s\in[t,T)\;,\quad U_T ~ = ~ g(X^{t,x}_T)
 \enq
 and 
 \beq\label{defsursolV}
  V_s & = & 0\;,\quad s\in[t,T]\;,
 \enq
 for $(t,x)\in [0,T]\times\R^d$. We therefore deduce from Theorem 4.2 in \cite{Peng99} that there exists a unique minimal solution $(Y^{t,x},Z^{t,x},K^{t,x})$ to \reff{EDSRC1}-\reff{EDSRC2} : for any other solution  $(U,V,A)$ to  \reff{EDSRC1}-\reff{EDSRC2}  we have
 \beqs
 Y^{t,x}_s & \leq & U_s\;,\quad s\in [t,T]\;.
 \enqs 
 The aim of this paper is to provide a numerical approximation of this minimal solution  $(Y^{t,x},Z^{t,x},K^{t,x})$. 

 \subsection{Related value function}
Since $Y^{t,x}$ is $\F^t$-adapted, $Y^{t,x}_t$ is almost surely constant and we can define the function 
   $v:~[0,T]\times\R^d\rightarrow\R$ by
\beqs
v(t,x) & = & Y_t^{t,x}\;,\quad (t,x)\in [0,T]\times \R^d\;.
\enqs 
From the uniqueness of the minimal solution to \reff{EDSRC1}-\reff{EDSRC2}, we have 
\beqs
Y_s^{t,x} & = & v(s,X^{t,x}_s)
\enqs 
for all $(t,x)\in [0,T]\times \R^d$ and $s\in[t,T]$.\\
 
 The aim of this paper is to provide a numerical approximation of this minimal solution  $(Y^{t,x},Z^{t,x},K^{t,x})$ or equivalently an approximation of the function $v$.
 
 We end this section by providing some properties of the function $v$. To this end, we define the facelift operator $F_\Cc$ defined by
\beqs
F_\Cc[\varphi] (x) & = & \sup_{y\in\R^d} \{\varphi(x+y)-\delta_\Cc(y)\} \;,\quad x\in\R^d\;,
\enqs
for any function $\varphi:~\R^d\rightarrow\R$, where $\delta_\Cc$ is the support function of the convex set $\Cc$
 \beqs
\delta_\Cc(y) & = & \sup_{z\in\Cc}z.y\;,\quad y\in\R^d \;.
\enqs 
We recall that $\delta_\Cc$ is positively homogeneous and convex. As a consequence the facelift operator $F_\Cc$ satifies
\beq\label{ident-facelift}
F_\Cc[F_\Cc[\varphi]] & = & F_\Cc[\varphi]
\enq
for any function $\varphi:~\R^d\rightarrow\R$. 

We  have the following properties for the function $v$.
 \begin{Proposition}\label{prop-reg-v}
 The function $v$ is locally bounded and satisfies the following properties.
 \begin{enumerate}[(i)]
 \item Time space regularity: there exists a constant $L$ such that
 \beq\label{reg-v}
 |v(t,x)-v(t',x')| & \leq & L\big(|t-t'|^{1\over 2}+|x-x'|\big)
 \enq
 for all $t,t'\in [0,T)$ and $x,x'\in \R^d$.
\item Facelift identity
\beq\label{facelift-id}
v(t,x) & = & F_\Cc[v(t,.)](x)
\enq
for all $(t,x)\in[0,T)\times \R^d$.
\item Value at $T^-$:
\beq\label{tarlouzeno36}
\lim_{t\rightarrow T^-}v(t,x) & = & F_\Cc[g](x)
\enq
for all $x\in\R^d$.
\end{enumerate}
 \end{Proposition}
 \begin{proof} These results mainly relie on \cite{BEM18}.
From classical estimates on BSDEs and the supersolution exhibited in \reff{defsursolU}-\reff{defsursolV}, the function $v$ is bounded. 
 
 The property \reff{reg-v} is a direct consequence of Theorem 2.1 (a) in \cite{BEM18}. 
 We turn to the facelift identity.  Fix $t\in[0,T)$, $\eps>0$ such that $t+\eps<T$ and $x\in\R^d$. Since  $(Y^{t,x},Z^{t,x},K^{t,x})$ is the minimal solution to  \reff{EDSRC1}-\reff{EDSRC2}, its restriction to $[t,t+\eps]$ is also the minimal solution to 
 \beqs
U_s & = & v(t+\eps,X^{t,x}_{t+\eps})+\int_s^{t+\eps}f(u,X^{t,x}_u,U_u,V_u)du-\int_s^{t+\eps} Z^{t,x}_ud B_u+A_T-A_s\quad \\
V_s & \in & \sigma(s,X_s^{t,x})^\top \Cc
\enqs
for $s\in[t,t+\eps]$\;. From Theorem  2.1 (b) and (c) in \cite{BEM18}, we deduce that
\beqs
v(t+\eps,X^{t,x}_{t+\eps}) & = & F_\Cc[v(t+\eps,.)](X^{t,x}_{t+\eps})\;.
\enqs
From \reff{reg-v}, we get \eqref{facelift-id} by sending $\eps$ to 0.
The last property is a consequence of  Theorem  2.1  (c) in \cite{BEM18}.
 \end{proof}

We end this section by a characterization of the minimal solution as the limit of penalized solutions. More precisely, we introduce the sequence $(Y^{n,t,x},Z^{n,t,x})\in \mathbf{S}^2_{[t,T]}\times \mathbf{H}^2_{[t,T]}$ which is defined for any $n \in \N^*$ as the solution of the following BSDE
\beq\nonumber
Y^{n,t,x}_s & = & g(X^{t,x}_T)\\ \nonumber& &+\int_s^T\Big(f(u,X^{t,x}_u,Y^{n,t,x}_u,Z^{n,t,x}_u) +n\max \big\{-H(\sigma^\top(u,X^{t,x}_u)^{-1}Z^{n,t,x}_u),0 \big\}\Big)du\\ & &-\int_s^T Z^{n,t,x}_ud B_u \;,\qquad s\in[t,T]\;, \label{EDSRpen}
\enq
for $(t,x)\in[0,T]\times \R^d$, where the operator $H$ is defined by
\beqs
H(p) & = & \inf_{|y|=1} (\delta_\Cc(y)-yp)\;,\quad p\in\R^d\;.
\enqs
We also introduce the related sequence of penalized PDEs
\begin{equation}\label{EDP penalisee}\left\{
\begin{array}{l}
-\partial_t  v_n(t,x)-\Lc  v_n(t,x) 
-f\big(t,x, v_n(t,x),\sigma(t,x)^\top D v_n(t,x)\big)\\
-n\max\{-H(D v_n(t,x)),0\}  =  0\;,~~  (t,x)\in [0,T)\times\R^d\\
 v_n(T,x)  =  g(x)\;,~~ x\in\R^d\\
\end{array}\right.
\end{equation}
where the second order local operator $\Lc$ related to the diffusion process $X$ is defined by
\beqs
\Lc \varphi (t,x) & = & b(t,x).D\varphi(t,x)+{1\over 2}\textrm{Tr}\big(\sigma\sigma^\top(t,x) D^2\varphi(t,x)\big) \;,\quad (t,x)\in[0,T]\times \R^d\;,
\enqs
for any function $\varphi:~[0,T]\times\R^d\rightarrow\R$ which is twice differentiable w.r.t. its space variable. As we use the notion of viscosity solution, we refer to \cite{CrandallIshiiLions} for its definition.
\begin{Proposition}\label{prop-pen-fct} (i)  For $(t,x)\in[0,T]\times\R^d$ the BSDE \reff{EDSRpen} admits a unique solution $(Y^{n,t,x},Z^{n,t,x})\in  \mathbf{S}^2_{[t,T]}\times \mathbf{H}^2_{[t,T]}$ and we have
\beqs
Y^{n,t,x}_s & = & v_n(s, X^{t,x}_s)\;,\quad s\in[t,T]\;,
\enqs
where $v_n$ is the unique viscosity solution to \reff{EDP penalisee} with polynomial growth.

\noindent (ii) The sequence  $(Y^{n,t,x})_{n \geq 1}$ is nondecreasing and
\beqs
\lim_{n\rightarrow+\infty}Y^{n,t,x}_s & = & Y^{t,x}_s\;,~ \P-a.s. \mbox{ for } s\in[t,T]\;,
\enqs
for any $(t,x)\in[0,T]\times\R^d$.

\noindent (iii) The sequence  $(v_n)_{n \geq 1}$ is nondecreasing and converges pointwisely to the function $v$ on $[0,T]\times\R^d$. 
\end{Proposition}
\begin{proof}
(i) By the definition of the operator $H$, the driver of BSDE 
\reff{EDSRpen} is globally Lipschitz continuous for all $n\geq1$. From Theorem 1.1 in \cite{pardoux1998backward}, there exists a unique  $(Y^{n,t,x},Z^{n,t,x})\in  \mathbf{S}^2_{[t,T]}\times \mathbf{H}^2_{[t,T]}$ solution to \reff{EDSRpen} for all $n\geq 1$.
Then from Theorem 2.2 in  \cite{pardoux1998backward}, we get that the function $v_n:~[0,T]\times \R^d\rightarrow\R^d$ defined by
\beqs
v_n(t,x) & = & Y^{n,t,x}_t\;,\quad (t,x)\in[0, T]\times\R^d\;, 
\enqs
is a continuous viscosity solution to \reff{EDP penalisee}. By uniqueness to BSDE \reff{EDSRpen}, we get
\beqs
Y^{n,t,x}_s & = & v_n(s, X^{t,x}_s)\;,\quad s\in[t,T]\;.
\enqs
Using Theorem 5.1 in \cite{PPR97}, $v_n$ is the unique  viscosity solution to \reff{EDP penalisee} with polynomial growth.

\vspace{2mm}

\noindent (ii) From Theorem 4.2 in \cite{Peng99}, the sequence  $(Y^{n,t,x})_{n \geq 1}$ is nondecreasing and converges pointwisely to   $\tilde Y^{t,x}$ where  $(\tilde Y^{t,x}, \tilde Z^{t,x})$ is the minimal solution to \reff{EDSRC1}, with the constraint
\beqs
H\big(\sigma^\top(X^{t,x})^{-1}\tilde Z^{t,x}\big) & \geq & 0\;,
\enqs 
 for any $(t,x)\in[0,T]\times\R^d$. Since $\Cc$ is closed, we get from  Theorem 13.1 in \cite{Rock70} 
\beqs
\tilde Z^{t,x} & \in & \sigma(X^{t,x})^\top \Cc
\enqs
and $(\tilde Y^{t,x}, \tilde Z^{t,x})=( Y^{t,x},  Z^{t,x})$ for all  $(t,x)\in [0,T)\times\R^d$.

\vspace{2mm}

\noindent (iii) The  nondecreasing convergence of $v_n$ to $v$ is an immediate consequence of (ii).
\end{proof}

\section{Discrete-time approximation of the constraint}\label{sec3}
\subsection{Discretely constrained BSDE}
We introduce in this section a BSDE with discretized constraint on the gains process. 
To this end, we first extend the definition of the facelift operator to random variables. More precisely, for $s\in[0,T]$ and  $L>0$, we denote by $\mathbf D^{}_{L,s}$ the set of random flows $R=\{R^{t,x},~(t,x)\in[0,s]\times \R^d\}$ of the form
 \beq\label{dom-def-FL}
 R^{t,x} & = & \varphi(X^{t,x}_{s})\;,\quad (t,x)\in[0,s]\times \R^d\;,
 \enq
where $\varphi:~\R^d\rightarrow\R$ is $L$-Lipschitz continuous. We also define the set $\mathbf D^{}_{s}$ by
\beqs
\mathbf D^{}_{s} & = & \bigcup_{L>0}\mathbf D^{}_{L,s}\;.
\enqs 
We then define the operator $\mathfrak{F}_{\Cc,s}$ on $\mathbf D^{}_{s}$ by
\beqs
\mathfrak{F}_{\Cc,s}[R]^{t,x} & = & F_\Cc[\varphi](X^{t,x}_{s})\;, \quad (t,x)\in[0,s]\times \R^d\;,
\enqs 
for $R\in \mathbf D^{}_{s}$ of the form \reff{dom-def-FL}.
We notice that the function $\varphi$ appearing in the representation \reff{dom-def-FL} is uniquely defined. Therefore the extended facelift operator $\mathfrak F _\Cc$ is well defined. Moreover, it satisfies the following stability property
\beq\label{stab-D-FL}
R\in \mathbf D^{}_{L,s} & \Rightarrow & \mathfrak{F}_{\Cc,s}^{}[R]\in \mathbf D^{}_{L,s}
\enq
for all $s\in[0,T]$, $L>0$  and  $(t,x)\in[0,T]\times\R^d$. Hence $\mathfrak F _{\Cc,s}$ maps $\mathbf D^{}_{s}$ into itself.

\vspace{2mm}

%
We then fix a grid $\mathcal R=\{r_0=0<r_1<\ldots<r_n=T\}$, with $n \in \N^*$, of the time interval $[0,T]$ and we consider the discretely constrained 
BSDE: find $(Y^{\mathcal{R},t,x},\tilde Y^{\mathcal{R},t,x},Z^{\mathcal{R},t,x},K^{\mathcal{R},t,x})\in  \mathbf S ^2 _{[t,T]}\times \mathbf S ^2 _{[t,T]}\times \mathbf H ^2 _{[t,T]}\times \mathbf A ^2 _{[t,T]}$ such that
\beq\label{EDSRDC1}
Y^{\mathcal{R},t,x}_T ~=~\tilde Y_T^{\mathcal{R},t,x} & = & F_\Cc[g](X_T^{t,x})
\enq
and
\beq\label{EDSRDC2}
\tilde Y^{\mathcal{R},t,x}_u & = & Y^{\mathcal{R},t,x}_{r_{k+1}}+\int_u^{r_{k+1}}f(s,X_s,\tilde Y^{\mathcal{R},t,x}_s,Z^{\mathcal{R},t,x}_s)ds-\int_u^{r_{k+1}}Z_s^{\mathcal{R},t,x} dB_s~~\qquad\quad \\
 Y^{\mathcal{R},t,x}_u & = & \tilde Y^{\mathcal{R},t,x}_u\mathds{1}_{(r_k,r_{k+1})}(u)+\mathfrak{F}_{\Cc,r_k}[\tilde Y^{\mathcal{R}}_u]^{t,x}\mathds{1}_{\{ r_k\} }(u)\label{EDSRDC3}
\enq
for $u\in[r_k,r_{k+1})\cap[t,T]$, $k=0,\ldots,n-1$, and
\beqs
K^{\mathcal{R},t,x}_u & = & \sum_{k=0}^n ( Y^{\mathcal{R},t,x}_{r_{k}}-\tilde Y^{\mathcal{R},t,x}_{r_{k}})\mathds{1}_{t\leq r_k\leq u\leq T}
\enqs
for $u\in[t,T]$.

We also introduce  the related PDE which takes the following form
\beq\label{EDPDC1}
v^{\mathcal{R}}(T,x) ~~=~~\tilde v^{\mathcal{R}}(T,x) & = & F_\Cc[g](x)\;,\quad x\in\R^d\;,
\enq
\begin{equation}\label{PDE-discrete-const}\left\{
\begin{array}{l}
-\partial_t \tilde v^{\mathcal{R}}(t,x)-\Lc \tilde v^{\mathcal{R}}(t,x) \\
-f\big(t,x,\tilde v^{\mathcal{R}}(t,x),\sigma(t,x)^\top D\tilde v^{\mathcal{R}}(t,x)\big)  =  0\;,~~  (t,x)\in [r_k,r_{k+1})\times\R^d\\
\tilde v^{\mathcal{R}}(r_{k+1}^-,x)  =  F_\Cc[\tilde v^{\mathcal{R}}(r_{k+1},.)](x)\;,~~ x\in\R^d\\
\end{array}\right.
\end{equation}
and 
\beq\label{EDPDC3}
v^{\mathcal{R}}(t,x) & = & \tilde v^{\mathcal{R}}(t,x)\mathds{1}_{(r_k,r_{k+1})}(t)+F_\Cc[\tilde v^{\mathcal{R}}(t,.)](x)\mathds{1}_{\{t=r_k\}}
\enq
for $(t,x)\in [r_k,r_{k+1})\times\R^d$, and $k=0,\ldots,n-1$.

 We first show the well-posedness of the BSDE  \reff{EDSRDC1}-\reff{EDSRDC2}-\reff{EDSRDC3}  and PDE \reff{EDPDC1}-\reff{PDE-discrete-const}-\reff{EDPDC3}, then we derive some regularity properties about the solutions.
\begin{Proposition}\label{reg-v-pi}
(i) For $(t,x)\in[0,T]\times\R^d$, the discretely constrained BSDE \reff{EDSRDC1}-\reff{EDSRDC2}-\reff{EDSRDC3} admits a unique solution $(Y^{\mathcal{R},t,x},\tilde Y^{\mathcal{R},t,x},Z^{\mathcal{R},t,x},K^{\mathcal{R},t,x})\in \mathbf S ^2 _{[t,T]}\times \mathbf S ^2 _{[t,T]}\times \mathbf H ^2 _{[t,T]} \times \mathbf A ^2 _{[t,T]}$. 

\noindent (ii)  The PDE  \reff{EDPDC1}-\reff{PDE-discrete-const}-\reff{EDPDC3} admits a unique bounded viscosity solution $(v^{\mathcal{R}},\tilde v^{\mathcal{R}})$ and we have
\beqs
Y^{\mathcal{R},t,x}_s  =  v^{\mathcal{R}}({s,X^{t,x}_s}) & \mbox { and  }~ \tilde Y^{\mathcal{R},t,x}_s  = \tilde  v^{\mathcal{R}}({s,X^{t,x}_s})\;,\quad s\in[t,T]\;,
\enqs
for $(t,x)\in[0,T)\times \R^d$. 

\noindent (iii) The family of functions $(v^{\mathcal{R}})_{\mathcal{R}}$  (resp.  $(\tilde v^{\mathcal{R}})_{\mathcal{R}}$)  is uniformly Lipschitz continuous in the space variable: there exists a constant $L$ such that
\beqs
|v^{\mathcal{R}}(t,x)-v^{\mathcal{R}}(t,x')| & \leq & L|x-x'|
\enqs
for all $t\in[0,T]$ and $x,x'\in\R^d$.

\noindent (iv) The family of functions $(v^{\mathcal{R}})_{\mathcal{R}}$  (resp.  $(\tilde v^{\mathcal{R}})_{\mathcal{R}}$) is uniformly ${1\over 2} $-H\"older left-continuous (resp. right-continuous) in the time variable: there exists a constant $L$ such that
\beqs
|v^{\mathcal{R}}(t,x)-v^{\mathcal{R}}(r_{k+1},x)| & \leq & L\sqrt{r_{k+1}-t}\\
\text{ (resp. }|\tilde v^{\mathcal{R}}(t,x)-\tilde v^{\mathcal{R}}(r_{k},x)| & \leq & L\sqrt{t-r_{k}} \text{ )}
\enqs
for all ${\mathcal{R}}=\{r_0=0,r_1,\ldots,r_n=T\}$ of $[0,T]$, $t\in(r_k,r_{k+1}]$ (resp. $t\in[r_k,r_{k+1})$),  $k=0,\ldots,n-1$ and $x,x'\in\R^d$.
\end{Proposition}
\begin{proof}We fix a grid $\mathcal R=\{r_0=0<r_1<\ldots<r_n=T\}$ of the time interval $[0,T]$.

\textit{Step 1. Existence and uniqueness to the BSDE and link with the PDE}. We  prove  by a backward induction on $k$ that \reff{EDSRDC2}-\reff{EDSRDC3} admits a unique solution
 on $[r_k,r_{k+1}]$ and that $\tilde Y ^{\mathcal \Rc,t,x}_{r_k}, Y ^{\mathcal \Rc,t,x}_{r_k}\in \mathbf D^{}_{s}$ and that 
 \beqs
 Y^{\mathcal{R},t,x}  =  v^{\mathcal{R}}({.,X^{t,x}}) & \mbox { and  }~ \tilde Y^{\mathcal{R},t,x}  = \tilde  v^{\mathcal{R}}({.,X^{t,x}})
 \enqs
 with $(v^{\mathcal{R}},\tilde v^{\mathcal{R}})$ the unique viscosity solution to \reff{EDPDC1}-\reff{PDE-discrete-const}-\reff{EDPDC3} with polynomial growth.
 
 \vspace{2mm}

$\bullet$ $k=n-1$. Since $g$ is Lipschitz continuous, it is the same for $F_\Cc[g]$. From \textbf{(H$b,\sigma$)} and \textbf{(H$f,g$)} the BSDE admits a unique solution (see e.g. Theorem 1.1 in \cite{pardoux1998backward}). From Theorem 2.2 in  \cite{pardoux1998backward}, the functions $(v^{\mathcal{R}},\tilde v^{\mathcal{R}})$ defined by
\beqs
 v^{\mathcal{R}}(t,x)  = Y^{\mathcal{R},t,x}_t  & \mbox { and  }~ \tilde  v^{\mathcal{R}}({t,x}) = \tilde Y^{\mathcal{R},t,x}_t  \;,\quad (t,x)\in[0,T)\times \R^d\;,
\enqs
 are the unique  viscosity solution to \reff{EDPDC1}-\reff{PDE-discrete-const}-\reff{EDPDC3} with polynomial growth.
 From the uniqueness to Lipschitz BSDEs (see e.g. Theorem 1.1 in \cite{pardoux1998backward}) we get \beqs
Y^{\mathcal{R},t,x}_s  =  v^{\mathcal{R}}({s,X^{t,x}_s}) & \mbox { and  }~ \tilde Y^{\mathcal{R},t,x}_s  = \tilde  v^{\mathcal{R}}({s,X^{t,x}_s})\;,\quad s\in[t,T]\;,
\enqs
for $(t,x)\in[r_{n-1},r_n)\times \R^d$.
  Then, from Proposition \ref{prop-reg-space-EDPSL}, we have $\tilde Y ^{\mathcal \Rc,t,x}_{r_n}\in \mathbf D^{}_{s}$. By \reff{stab-D-FL}, $Y ^{\mathcal \Rc,t,x}_{r_{n-1}}\in \mathbf D^{}_{s}$

$\bullet$ Suppose the property holds for $k+1$. Then $\tilde Y ^{\mathcal \Rc,t,x}_{r_{k+1}}\in \mathbf D^{}_{s}$.  From Theorem 1.1 in \cite{pardoux1998backward}, we get the existence and uniqueness of the solution on $[r_{k},r_{k+1}]$. Then, from Theorem 2.2 in  \cite{pardoux1998backward} and Theorem 5.1 in \cite{PPR97}, the functions $(v^{\mathcal{R}},\tilde v^{\mathcal{R}})$ defined by
\beqs
 v^{\mathcal{R}}(t,x)  = Y^{\mathcal{R},t,x}_t  & \mbox { and  }~ \tilde  v^{\mathcal{R}}({t,x}) = \tilde Y^{\mathcal{R},t,x}_t  \;,\quad (t,x)\in[r_k,r_{k+1})\times \R^d\;,
\enqs
 are the unique viscosity solution to \reff{EDPDC1}-\reff{PDE-discrete-const}-\reff{EDPDC3} with polynomial growth.
From The uniqueness to Lipschitz BSDEs (see e.g. Theorem 1.1 in \cite{pardoux1998backward}) we get \beqs
Y^{\mathcal{R},t,x}_s  =  v^{\mathcal{R}}({s,X^{t,x}_s}) & \mbox { and  }~ \tilde Y^{\mathcal{R},t,x}_s  = \tilde  v^{\mathcal{R}}({s,X^{t,x}_s})\;,\quad s\in[t,T]\;,
\enqs
for $(t,x)\in[r_k,r_{k+1})\times \R^d$.

 From Proposition \ref{prop-reg-space-EDPSL}  we also have $\tilde Y ^{\mathcal \Rc,t,x}_{r_{k}}\in \mathbf D^{}_{s}$. By \reff{stab-D-FL}, $Y ^{\mathcal \Rc,t,x}_{r_{k}}\in \mathbf D^{}_{s}$.

\vspace{2mm}


\noindent \textit{Step 2. Uniform space Lipschitz continuity.} From the definition of the function $v^\Rc$, \reff{stab-D-FL} and Proposition \ref{prop-reg-space-EDPSL}, we get  a backward induction on $k$ that 
\beqs
|v^{\mathcal{R}}(t,x)-v^{\mathcal{R}}(t,x')| & \leq & L_k|x-x'|
\enqs
for all $t\in[r_k,r_{k+1}]$ and $x,x'\in\R^d$ with 
\beqs
L_n & = & L_g
\enqs
and
\beqs
L_{k-1} & = & e^{C(r_k - r_{k-1})}(1+(r_k - r_{k-1}))^{1\over 2}\big( L_k^2+C(r_k - r_{k-1})\big)^{1\over 2} 
\enqs
for $k=0,\ldots,n-1$, where $C= 2L_{b,\sigma}+L_{b,\sigma}^2+(L_f\vee 2)^2$.
We therefore get
\beqs
L^2_{k} & = & L^2_g\prod_{j=k}^{n-1}(1+r_{j+1}-r_{j})e^{2C(r_{j+1}-r_{j})}\\
 & & + \sum_{\ell=k}^{n-1}C(r_{\ell+1}-r_{\ell})\prod_{j=k}^\ell e^{2C(r_{j+1}-r_{j})}(1+r_{j+1}-r_{j})\\
 & \leq & L_g^2\prod_{j=0}^{n-1}(1+r_{j+1}-r_{j}) +CTe^{CT}\prod_{j=0}^{n-1}(1+r_{j+1}-r_{j})\\
 & \leq & \Big(L_g^2+CTe^{2CT}\Big)\Big(1+{T\over n}\Big)^{n}
\enqs
for $k=0,\ldots,n-1$. Since the sequence $\Big(\big(1+{T\over n})^{n}\Big)_{n\geq 1}$ is bounded  we get the  space Lipschitz property uniform in the grid $\Rc$.
%
%

\vspace{2mm}

\noindent \textit{Step 3. Uniform time H\"older continuity.} From the previous step and Proposition \ref{prop-reg-time-EDPSL}, we get the H\"older regularity uniform in the grid $\Rc$.

\end{proof}

\subsection{Convergence of the discretely constrained BSDE}
We fix a sequence $(\mathcal R ^n)_{n\geq1}$ of grids of the time interval $[0,T]$ of the form
\beqs
\mathcal{R}^n & := & \big\{r_0^n=0<r_1^n<\cdots<r_{\kappa_n}^n=T\big\}\;,\quad n\geq 1\;.
\enqs
We suppose this sequence is nondecreasing, that means $\mathcal R ^n\subset\mathcal R ^{n+1}$ for $n\geq 1$, and  
\beqs
|\mathcal R^n|~~:=~~\max_{1\leq k\leq \kappa_n}(r^n_{k}-r^n_{k-1}) & \xrightarrow[n\rightarrow+\infty]{} & 0\;.
\enqs
\begin{Theorem}\label{Thm-conv-func}
The sequences of functions $(\tilde v^{\mathcal{R}^n})_{n\geq 1}$ and $(v^{\mathcal{R}^n})_{n\geq 1}$
are nondecreasing and converges to $v$
\beqs
\lim_{n\rightarrow +\infty} v^{\mathcal{R}^n}(t,x) & = & \lim_{n\rightarrow +\infty} \tilde v^{\mathcal{R}^n}(t,x) ~~=~~v(t,x)
\enqs
for all $(t,x)\in [0,T)\times\R^d$. \end{Theorem}
To prove Theorem \ref{Thm-conv-func} we need the following Lemma.
\begin{Lemma}\label{lem-1-edp}
Let $u:~[0,T]\times \R^d\rightarrow\R$ be a locally bounded function such that
\beq\label{cond-facelift}
u(t,x) & = & F_\Cc[u(t,.)](x)
\enq
for all $(t,x)\in[0,T]\times\R^d$.  Then u is a viscosity supersolution to
\beqs
H( Du) & = & 0 \;.
\enqs 
\end{Lemma}
\begin{proof} Fix $(\bar t,\bar x)\in[0,T]\times\R^d$ and $\varphi\in C^{1,2}([0,T]\times\R^d)$ such that 
\beqs
0~~=~~(u-\varphi)(\bar t,\bar x) & = & \min_{[0,T]\times\R^d} (u-\varphi) (t,x)\;.
\enqs
From \reff{cond-facelift} we get 
\beq\label{cond-facelift-test-function}
\varphi(\bar t,\bar x) & = & F_\Cc[\varphi(\bar t,.)](\bar x) \;.
\enq
Fix $y\in \Cc$. From Taylor formula we have
\beqs
\varphi(\bar t,\bar x+y) & = & \varphi(\bar t,\bar x)+\int_0^1D\varphi\big(\bar t,s\bar x+(1-s)(\bar x+y)\big).yds
\enqs
Since $0\in \Cc$ we have $\eps y\in \Cc$ for any $\eps\in(0,1)$. Since $\delta_\Cc$ is positively homogeneous, we get by taking $\eps y$ in place of $y$ 
\beqs
\eps\Big(\delta_\Cc(y)-\int_0^1D\varphi\big(\bar t,\bar x+(1-s)\eps y\big).yds\Big)  & = & \varphi(\bar t, \bar x)-\big(\varphi(\bar t , \bar x+\eps y)-\delta_\Cc(\eps y)\big) \;.
\enqs
Then from \reff{cond-facelift-test-function} we get
\beqs
\delta_\Cc(y)-\int_0^1D\varphi\big(\bar t,\bar x+(1-s)\eps y\big).yds & \geq & 0
\enqs
for all $\eps>0$. Since $\varphi\in C^{1,2}([0,T]\times\R^d)$, we can apply the dominated convergence theorem and we get by sending $\eps$ to $0$ 
\beqs
\delta_\Cc(y)-D\varphi(\bar t,\bar x).y & \geq & 0\;.
\enqs
Since $y$ is arbitrarily chosen in $\Cc$ we get
\beqs
H\big( D\varphi(\bar t, \bar x)\big) & \geq & 0\;.
\enqs
\end{proof}

\begin{proof}[Proof of Theorem \ref{Thm-conv-func}] Fix $(t,x)\in[0,T]\times \R^d$. Since the sequence of grids $(\mathcal R ^n)_{n\geq 1}$ is nondecreasing and $\mathfrak{F}_{\Cc,s}[Y]\geq Y$ for any $Y\in \mathbf D_s$, using the comparison Theorem 2.2 in \cite{EKPQ97}, we get by induction that the sequences $( Y^{\mathcal{R}^n,t,x})_{n\geq 1}$ and $(\tilde Y^{\mathcal{R}^n,t,x})_{n\geq 1}$ are nondecreasing. Therefore the sequences of functions $(\tilde v^{\mathcal{R}^n})_{n\geq 1}$ and $(v^{\mathcal{R}^n})_{n\geq 1}$
are nondecreasing and we can define the limits 
\beqs
w(t,x) & = & \lim_{n\rightarrow +\infty} v^{\mathcal{R}^n}(t,x)\\
\tilde w(t,x) & = & \lim_{n\rightarrow +\infty} \tilde v^{\mathcal{R}^n}(t,x) 
\enqs
for all $(t,x)\in [0,T]\times\R^d$.
We proceed in four steps to prove that 
$v=w=\tilde w$.

\vspace{2mm}

\noindent \textit{Step 1. We have $\tilde w=w\leq v$.} Still using  the comparison Theorem 2.2 in \cite{EKPQ97} we get by induction 
\beqs
 w(t,x) & \geq & \tilde w(t,x)\;,\quad (t,x)\in[0,T]\times \R^d\;.
\enqs
Moreover, we get from Proposition \ref{prop-reg-v} that $(Y^{t,x}\mathds{1}_{[t,T)}+F_\Cc[g](X^{t,x}_T)\mathds{1}_{\{T\}},Z^{t,x})$ is a continuous supersolution to \reff{EDSRDC2} on each interval $[r_k^n,r_{k+1}^n]\cap[t,T]$. Therefore, using Remark b. of Section 2.3 in \cite{EKPQ97}, we get by induction  $Y^{t,x}\geq Y^{\mathcal{R}^n,t,x}$ for all $n\geq 1$.
Hence  
\beqs
v(t,x) & \geq & w(t,x)
\enqs
for all  $(t,x)\in [0,T)\times\R^d$. 
We now prove $w=\tilde w$. Fix $n\geq 1$, $k\in \{0,\ldots,\kappa_{n}-1\}$, $t\in[r^n_k,r^n_{k+1})$ and $x\in\R^d$. We have 
\beqs
| v^{\mathcal{R}^n}(t,x)-\tilde v^{\mathcal{R}^n}(t,x)| & \leq & |v^{\mathcal{R}^n}(t,x)- v^{\mathcal{R}^n}(r^n_{k+1},x)|+| v^{\mathcal{R}^n}(r^n_{k+1},x)-\tilde v^{\mathcal{R}^n}(r^n_{k},x)|\\
 & & +| \tilde v^{\mathcal{R}^n}(r^n_k,x)-\tilde v^{\mathcal{R}^n}(t,x)| \;.
\enqs
From Proposition \ref{reg-v-pi} (iii) we get
\beqs
| v^{\mathcal{R}^n}(t,x)-\tilde v^{\mathcal{R}^n}(t,x)|  & \leq & 2L\sqrt{|{\mathcal{R}^n}|}+|v^{\mathcal{R}^n}(r^n_{k+1},x)-\tilde v^{\mathcal{R}^n}(r^n_{k},x)| \;.
\enqs
Since $v^{\mathcal{R}^n}$ coincides with $\tilde v^{\mathcal{R}^n}$ out of the grid ${\mathcal{R}^n}$ we  have
\beqs
|v^{\mathcal{R}^n}(r^n_{k+1},x)-\tilde v^{\mathcal{R}^n}(r^n_{k},x)| & \leq & |v^{\mathcal{R}^n}(r^n_{k+1},x)- v^{\mathcal{R}^n}({r^n_{k}+r^n_{k+1}\over 2},x)|\\
 & & +|\tilde v^{\mathcal{R}}({r^n_{k}+r^n_{k+1}\over 2},x)-\tilde v^{\mathcal{R}}(r^n_{k},x)|\;.
\enqs
Still using Proposition \ref{reg-v-pi} (iii)  we get
\beqs
|v^{\mathcal{R}^n}(r^n_{k+1},x)-\tilde v^{\mathcal{R}^n}(r^n_{k},x)| & \leq & 2L\sqrt{|{\mathcal{R}^n}|}
\enqs
and
\beqs
|v^{{\mathcal{R}^n}}-\tilde v^{{\mathcal{R}}^n}| & \leq & 4L \sqrt{|{\mathcal{R}^n}|}~~\xrightarrow[n\rightarrow+\infty]{}~~0\;.
\enqs

\vspace{2mm}

\noindent \textit{Step 2. The function $w$ satisfies
\beqs
w(t,x) & = & F_\Cc[w(t,.)](x)
\enqs
for all $(t,x)\in[0,T]\times\R^d$.}
 We first prove
\beqs
\lim_{n\rightarrow+\infty}F_\Cc[v^{\mathcal{R}^n}(t,.)](x)-v^{\mathcal{R}^n}(t,x) & = & 0 \;.
\enqs
Fix $n\geq 1$. If $t\in{\mathcal{R}^n}$,  then $F_\Cc[v^{\mathcal{R}^n}(t,.)]-v^{\mathcal{R}^n}(t,.)=0$ from \reff{ident-facelift}.  Fix now $k\in \{0,\ldots,\kappa_n-1\}$ and $t\in(r^n_k,r^n_{k+1})$. Then, still using  \reff{ident-facelift}, we have $v^{{\mathcal{R}}}(r^n_{k+1},.)=F_\Cc[v^{{\mathcal{R}}}(r^n_{k+1},.)]$. Therefore we get
\beqs
|F_\Cc[v^{\mathcal{R}^n}(t,.)]-v^{\mathcal{R}^n}(t,.)| & \leq & |F_\Cc[v^{\mathcal{R}^n}(r^n_{k+1},.)](.)-F_\Cc[v^{\mathcal{R}^n}(t,.)]|\\
 & & +|v^{\mathcal{R}^n}(r^n_{k+1},.)-v^{\mathcal{R}^n}(t,.)|\\
  & \leq & 2\sup_{x\in\R^d} |v^{\mathcal{R}^n}(r^n_{k+1},x)-v^{\mathcal{R}^n}(t,x)|\;.
\enqs
We deduce from Proposition \ref{reg-v-pi} (ii) that
\beqs
\sup_{x\in\R}|F_\Cc[v^{\mathcal{R}^n}(t,.)](x)-v^{\mathcal{R}^n}(t,x)| & \leq & 2L\sqrt{|{\mathcal{R}^n}|}~~\xrightarrow[n\rightarrow+\infty]{} ~~0\;.
\enqs
Then we have
\beqs
0~~\leq~~ F_\Cc[w(t,.)](x)-w(t,x)  & = &  F_\Cc\big[\lim_{n\rightarrow+\infty}v^{\mathcal{R}^n}(t,.)\big](x)-\lim_{n\rightarrow+\infty}v^{\mathcal{R}^n}(t,x) \\  & \leq & \lim_{n\rightarrow+\infty}F_\Cc[v^{\mathcal{R}^n}(t,.)](x)-v^{\mathcal{R}^n}(t,x) ~~=~~0\;.
\enqs
%
%
%


\vspace{2mm}

\noindent \textit{ Step 3. The function $w$ is a viscosity supersolution to
\begin{equation}\label{EDP-ss-pen}\left\{
\begin{array}{l}
-\partial_t  w(t,x)-\Lc  w(t,x) \\
-f\big(t,x, w(t,x),\sigma(t,x)D w(t,x)\big) =  0\;,~~  (t,x)\in [0,T)\times\R^d\;.\\
 w(T,x)  =  g(x)\;,~~ x\in\R^d\;,\\
\end{array}\right.
\end{equation}}
 We first prove that $v^{\mathcal{R}^n}$ is a viscosity supersolution to \eqref{EDP-ss-pen} for any $n\geq 1$. Fix  $(\bar t,\bar x)\in[0,T]\times\R^d$ and $n\geq 1$. If $\bar t=T$ then we have $v^{\mathcal{R}^n}(\bar t,\bar x)\geq g(\bar x)$. If $\bar t\notin{\mathcal{R}^n}$ we deduce the viscosity supersolution property from \reff{PDE-discrete-const}. Suppose now that $\bar t=r_k$ for some $k=0,\ldots,n-1$. Fix $\varphi\in C^{1,2}([0,T]\times\R^d)$ such that 
\beqs
0~~=~~(v^{{\mathcal{R}^n}}_*-\varphi)(\bar t,\bar x) & = & \min_{[0,T]\times\R^d} (v^{{\mathcal{R}^n}}_*-\varphi)\;.
\enqs
We observe that  the lsc envelope $v^{{\mathcal{R}^n}}_*$ of $v^{{\mathcal{R}^n}}$ is the function $\tilde v^{{\mathcal{R}^n}}$. We then have
\beqs
0~~=~~(\tilde v^{{\mathcal{R}^n}}-\varphi)(\bar t,\bar x) & = & \min_{[r^n_k,r^n_{k+1}]\times\R^d} (\tilde v^{{\mathcal{R}^n}}-\varphi)\;.
\enqs
From the viscosity property of $\tilde v^{{\mathcal{R}^n}}$, we deduce that
\beqs
\partial_t  \varphi(\bar t,\bar x)-\Lc  \varphi(\bar t,\bar x)
-f\big(\bar t,\bar x, \varphi(\bar t,\bar x),\sigma(\bar t,\bar x)D \varphi(\bar t,\bar x)\big) & \geq & 0
\enqs
and $v^{\mathcal{R}}$ is a viscosity supersolution. We now turn to $w$. Since 
$v^{\mathcal{R}^n}\uparrow w$ as $n\uparrow+\infty$,  we can apply stability results for semi-linear PDEs (see e.g. Theorem 4.1 in \cite{Bar94}) and we get the viscosity supersolution property of $w$. 

\vspace{2mm}

\noindent \textit{Step 4. We have $w= v$.} In view of Step 1, it sufficies to prove that  $w\geq v$. 
From Lemma \ref{lem-1-edp} and Step 2,  $w$ is a viscosity supersolution to $H\big( Dw\big)\geq 0$. Then from Step 3,  we deduce that $w$ is a viscosity supersolution to \reff{EDP penalisee}.
By Theorem 4.4.5 in \cite{Pham-book} we get $w\geq v_n$ for all $n\geq 1$ and hence $w\geq v$ from Proposition \ref{prop-pen-fct} (iii).
\end{proof}
\begin{Corollary}
We have the following uniform convergence
\beqs
\lim_{n\rightarrow +\infty} \sup_{(t,x)\in[0,T)\times Q}|v^{\mathcal{R}^n}(t,x)-v(t,x)| & = & \\
\lim_{n\rightarrow +\infty}\sup_{(t,x)\in[0,T)\times Q} |\tilde v^{\mathcal{R}^n}(t,x)-v(t,x)|  & = & 0
\enqs
for every compact subset $Q$ of $\R^d$.
\end{Corollary}
\begin{proof} We first define the function $\hat v$ by
\beqs
\hat v(t,x) & = & v(t,x)\mathds{1}_{[0,T)}(t)+F_\Cc[g](x)\mathds{1}_{\{T\}}(t)\;,\quad (t,x)\in[0,T]\times\R^d\;.
\enqs
From Proposition \ref{prop-reg-v}, $\hat v$ is continuous on $[0,T]\times\R^d$. Fix a compact $Q$ of $\R^d$. Using Dini's Theorem we get 
\beqs
\lim_{n\rightarrow +\infty} \sup_{x\in Q}|v^{\mathcal{R}^n}(t,x)-\hat v(t,x)| & = & \\\lim_{n\rightarrow +\infty}\sup_{x\in Q} |\tilde v^{\mathcal{R}^n}(t,x)-\hat  v(t,x)|  & = & 0
\enqs
for every  $t\in[0,T]$. 
In particular, if we define for $n\geq 1$ the functions $\Phi_n:~[0,T]\rightarrow\R$ by 
\beqs
\Phi_n (t) & = & \sup_{x\in Q}|v^{\mathcal{R}^n}(t,x)-\hat v(t,x)| \;,\quad t\in[0,T]\;,
\enqs
then 
$(\Phi_n)_{n \geq 1}$ is a nonincreasing sequence of \textit{c\`adl\`ag} functions such that
\beqs
\lim_{n\rightarrow +\infty}\Phi_n(t) & = & 0  
\enqs
for all $t\in[0,T]$ and
\beqs
\lim_{n\rightarrow +\infty}\Phi_n(t^-) & = &\lim_{n\rightarrow +\infty}\sup_{x\in Q}|\tilde v^{\mathcal{R}^n}(t,x)-\hat v(t,x)| ~~=~~0  
\enqs
for all $t\in(0,T]$. We then apply Dini's Theorem for \textit{c\`adl\`ag} functions (see the Lemma in the proof of Theorem 2 Chapter VII Section 1 in  \cite{DM78})
and we get the uniform convergence of $(\Phi_n)_{n \geq 1}$ to 0. Since $\hat v$ coincides with $v$ on $[0,T)\times\R^d$, we get the desired result.
\end{proof}
\begin{Corollary}\label{cor-cv-disc-const}
We have the following convergence result
\beqs
\lim_{n\rightarrow +\infty}\E\Big[\sup_{[t,T)}\big|Y^{\mathcal{R}^n,t,x}-Y^{t,x}\big|^2\Big]+\E\Big[\sup_{[t,T)}\big|\tilde Y^{\mathcal{R}^n,t,x}-Y^{t,x}\big|^2\Big] & & \\
+\E\Big[\int_{t}^T\big|Z_s^{\mathcal{R}^n,t,x}-Z_s^{t,x}\big|^2ds\Big]  & = & 0\;,
\enqs
for all $(t,x)\in[0,T)\times\R^d$.
\end{Corollary}
\begin{proof}
We first write 
\beqs
\sup_{s\in[t,T)}\big|Y_s^{\mathcal{R}^n,t,x}-Y_s^{t,x}\big|^2 & = & \sup_{s\in[t,T)} \big|v^{\mathcal{R}^n}(s,X^{t,x}_s)-v(s,X^{t,x}_s)\big|^2\;,\\
\sup_{s\in[t,T)}\big|\tilde Y_s^{\mathcal{R}^n,t,x}-Y^{t,x}_s\big|^2 & = & \sup_{s\in[t,T)} \big|\tilde v^{\mathcal{R}^n}(s,X^{t,x}_s)-v(s,X^{t,x}_s)\big|^2\;.
\enqs
Since $X$ has continuous paths, we get from Theorem \ref{Thm-conv-func}
\beqs
\lim_{n\rightarrow+\infty}\sup_{[t,T)}\big|Y^{\mathcal{R}^n,t,x}-Y^{t,x}\big|^2 +\sup_{[t,T)}\big|\tilde Y^{\mathcal{R}^n,t,x}-Y^{t,x}\big|^2& = & 0\;,\quad \P-a.s.
\enqs
By Lebesgue dominated convergence Theorem we get
\beqs
\lim_{n\rightarrow +\infty}\E\Big[\sup_{[t,T)}\big|Y^{\mathcal{R}^n,t,x}-Y^{t,x}\big|^2\Big]+\E\Big[\sup_{[t,T)}\big|\tilde Y^{\mathcal{R}^n,t,x}-Y^{t,x}\big|^2\Big]& = & 0\;.
\enqs
By classical estimates on BSDEs based on BDG and Young inequalities and Gronwall Lemma, we deduce 
\beqs
\lim_{n\rightarrow +\infty}\E\Big[\int_{t}^T\big|Z_s^{\mathcal{R}^n,t,x}-Z_s^{t,x}\big|^2ds\Big]  & = & 0\;.
\enqs
\end{proof}
\section{Neural network approximation of the discretely constrained BSDE}\label{sec4}
\subsection{Neural networks and approximation of the facelift}
We first recall the definition of a neural network with single hidden layer.
 To this end, we fix a function  $\rho:~\R^d\rightarrow \R$ called the activation function,  and an integer $m\geq1$, representing the number of neurons (also called nodes) on the hidden layer.  
\begin{Definition}
The set  $\mathfrak{N N}^\rho_m$ of feedforward neural network with single hidden layer with $m$ neurons and the activation function $\rho$ is the set of functions
\beqs
x \in\R^d & \mapsto & \sum_{i=1}^m\lambda_i\rho(\alpha_i.x)\in\R\;,
\enqs
where $\lambda_i\in\R$ and $\alpha_i\in\R^d$, for $i=1,\ldots,m$.
\end{Definition}
\noindent For $m\geq1$ we define the set $\Theta_m$ by
\beqs
\Theta_m & := & \Big\{(\lambda_i,\alpha_i)_{i=1,\ldots,m}~:~\lambda_i\in\R \mbox{ and }\alpha_i\in\R^{d}\mbox{ for }  {i=1,\ldots,m}\Big\}\;.
\enqs 
For $\theta=(\lambda_i,\alpha_i)_{i=1,\ldots,m}\in\Theta_m$, we denote by $NN^\theta$ the function from $\R^d$ to $\R$ defined by
\beqs
NN^\theta (x) & = & \sum_{i=1}^m\lambda_i\rho(\alpha_i.x)\in\R\;,\quad x\in\R^d\;.
\enqs
We also define the set $\mathfrak{N N}^\rho$ by
\beqs
\mathfrak{N N}^\rho & := & \bigcup_{m\geq 1} \mathfrak{N N}^\rho_m \;.
\enqs
We suppose in the sequel that $\rho$ is not identically equal to $0$, belongs to $C^1(\R,\R)$ and satisfies $\int_{\R}|\rho'(x)|dx<+\infty$. We denote by $C^1(\R^d,\R)$ the set functions in $C^1_b(\R^d,\R)$ with bounded derivative. We then have the following result from \cite{HSW90}.
\begin{Theorem}\label{ThmHornikdiff} $\mathfrak{N N}^\rho_{}$ is dense in $C^1_b(\R^d,\R)$
 for the topology of uniform convergence on compact sets: for any $f\in C^1_b(\R^d,\R)$ and for any compact $Q$ of $\R^d$, there exists a sequence $(NN^{\theta_\ell})_{\ell\geq1}$ of $\mathfrak{N N}^\rho$ such that
 \beqs
 \sup_{x\in Q}|NN^{\theta_\ell}(x)-f(x)|+ \sup_{x\in Q}|D NN^{\theta_\ell}(x)-Df(x)| & \xrightarrow[]{\ell\rightarrow+\infty} & 0\;.
 \enqs
\end{Theorem}


We turn to the facelift approximation by feedforward neural networks.
We fix  bounded and Lipschitz continuous functions $\varphi$ and $\varphi_\ell$, $\ell\geq1$, from $\R^d$ to $\R$ and a random variable $\xi$. For $\eps>0$, we define the sequence of parameters $(\theta_{m,\eps,\ell}^*)_{m,\eps,\ell}$ by
\beq\label{defThetastar}
\theta_{m,\eps,\ell}^* & \in & \textrm{arg}\min_{\theta\in \Theta_m}\E\Big[ \big|(NN^\theta-\varphi_\ell)(\xi)\big|^2\mathds{1}_{B_\eps}(\xi) \Big]\\
 & & \theta ~\textrm{ s.t.}~\P(D NN ^\theta(\xi)\in \Cc_\eps ;~(NN^\theta-\varphi_\ell)(\xi)\geq -\eps\;\big|\;\xi\in B_\eps)= 1\nonumber
\enq
where $D NN ^\theta$ denotes the gradient of $NN ^\theta$, $\Cc_\eps$ stands for the closed convex set defined by 
\beqs
\Cc_\eps & = & \Big\{y\in\R^d~:~\exists x\in \Cc,~|x-y|\leq \eps\Big\}\;,
\enqs
and $B_\eps$ stands for the ball $B(0,{1\over \eps})$.
\begin{Proposition}\label{prop-comp-facelift}
Suppose $\textrm{Supp}(\P_\xi)=\R^d$ and that $\E[|\xi|^2]<+\infty$. Then, if $\varphi_\ell$ converges uniformly to $\varphi$ on compact sets, we have
\beq\label{estim-approx-FL1}
\lim_{\eps\rightarrow0}\lim_{m\rightarrow+\infty}\lim_{\ell\rightarrow+\infty}\E\Big[ |(NN^{\theta^*_{m,\eps,\ell}}-F_\Cc[\varphi])(\xi)|^2\mathds{1}_{B_\eps}(\xi)\Big] & = & 0 \;.
\enq\label{estim-approx-FL2}
Moreover, we have
\beq
\lim_{\eps\rightarrow0}\lim_{m\rightarrow+\infty}\lim_{\ell\rightarrow+\infty}\E\Big[ |(NN^{\theta^*_{m,\eps,\ell}}\vee M\wedge(-M)-F_\Cc[\varphi])(\xi)|^2\Big] & = & 0 \;.
\enq
for any constant $M>0$ such that $|\varphi|\leq M$.
\end{Proposition}
To prove this theorem we need the following Lemma.
\begin{Lemma}\label{lem-convVA}
Let $X$ and $(X_n)_{n \geq 1}$ be  positive integrable random variables such that 
\beq\label{cond-suite-liminf}
\liminf_{n\rightarrow+\infty} X_n & \geq & X~~\geq~~0\;,
\enq
and
\beq\label{cond-con-esp}
\limsup_{n\rightarrow+\infty} \E[X_n] & \leq & \E[X]\;.
\enq
Then, we have
\beqs
X_n & \xrightarrow[\P-p.s.]{n\rightarrow+\infty} & X\;.
\enqs
\end{Lemma}
\begin{proof} We argue by contradiction. Suppose the $\P$-a.s. convergence of $X_n$ to $X$ does not hold. Then from \reff{cond-suite-liminf}, there exists some $\eta>0$ and $\Omega_\eta\subset\Omega$ such that $\P(\Omega_\eta)>0$ and 
\beq\label{cond-suite-liminf2}
\liminf_{n\rightarrow+\infty} X_n & \geq & X+\eta \quad ~~\mbox{ on }~ \Omega_\eta\;.
\enq
From Fatou's Lemma and \reff{cond-con-esp} we get
\beqs
\E[X] & \geq & \E\big[ \liminf_{n\rightarrow+\infty} X_n \big]~~\geq~~\E[X]+\eta\P(\Omega_\eta)
\enqs
which contradicts $\P(\Omega_\eta)>0$.
\end{proof}
\begin{proof}[Proof of Proposition \ref{prop-comp-facelift}] \textit{Step 1.} We prove that for any $\eps>0$, there exists a sequence 
 $(\theta'_{m,\eps})_{m\geq1}$ such that $\theta'_{m,\eps}\in\Theta_m$ for $m\geq1$, and 
\beqs
\limsup_{m\rightarrow+\infty}
\E\Big[|NN^{\theta'_{m,\eps}}(\xi)-F_\Cc[\varphi](\xi)|^2\;\big|\;\xi\in B_\eps\Big] & \leq & \eps\;, 
\enqs 
and  
\beqs
\P\Big(D NN^{\theta'_{m,\eps}} ~ \in  ~\Cc_\eps;~(NN^{\theta'_{m,\eps}}-\varphi_\ell)(\xi)\geq \eps\;\big|\;\xi\in B_\eps\Big) & = & 1\;,
\enqs
for $\ell$ and  $m$ large enough.

To this end, we introduce the sequence of mollifiers $\psi_n:~\R^d\rightarrow\R_+$, $n\geq 1$, defined by
 \beqs
 \psi_n(x) & := & {n^d\psi(nx)
 }\;, \qquad x\in\R\;,
 \enqs
 where the function $\psi\in C^\infty(\R^d,\R_+)$ has a compact support and is such that $\int_{\R^d}\psi(u)du=1$.
We then define the functions $\phi_{n}$, $n\geq 1$, by
\beqs
\phi_{n}(x) &: = & \int_{\R^d} \psi_n(y)F_\Cc[\varphi](x-y)dy\;,\quad x\in\R\;.
\enqs
Since $\varphi$ is Lipschitz continuous and bounded, $F_\Cc[\varphi]$ is also Lipschitz continuous and bounded. From classical results, we know that $\phi_n$ converges to $F_{\Cc}[\varphi]$ as $n$ goes to infinity uniformly on every compact subset of $\R$.
Moreover, $\phi_n\in C^\infty(\R^d,\R_+)$. Since $F_\Cc[\varphi]$ is Lipschitz continuous it is almost everywhere differentiable by Rademacher Theorem and we get from the dominated convergence Theorem
\beqs
D\phi_n(x) & = & \int_{\R^d} \psi_n(y)DF_\Cc[\varphi](x-y)dy \;.
\enqs 
From Lemma \ref{lem-1-edp}, we have  $DF_\Cc[\varphi]  \in  \Cc$ 
almost everywhere on $\R^d$. Since $\Cc$ is convex, we get
\beq\label{diff-mol-conv}
D\phi_n(x) & \in  & \Cc
\enq
for all $x\in\R^d$ and all $n\geq 1$.
Fix now $\eps>0$. Then there exists $n_\eps \in \N^*$ such that
\beq\label{estim-mol-FL}
\sup_{x\in B_\eps}\big|F_\Cc[\varphi](x)-\phi_{n_\eps}(x)\big| & \leq & {\eps\over 3}\;.
\enq
From Theorem \ref{ThmHornikdiff}, there exists a sequence $(\theta'_{m,\eps})_{m \geq 1}$ such that
\beq\label{convNN1}
\sup_{B_\eps}\big|NN^{\theta'_{m,\eps}}-\phi_{n_\eps}\big| 
+\sup_{B_\eps}\big|D NN^{\theta'_{m,\eps}}-D\phi_{n_\eps}\big| & \xrightarrow[m\rightarrow+\infty]{} & 0\;.
\enq
We therefore get from \reff{diff-mol-conv}, \reff{estim-mol-FL} and \reff{convNN1}
\beqs
\P\Big(D NN^{\theta'_{m,\eps}} ~ \in  ~\Cc_\eps;~(NN^{\theta'_{m,\eps}}-\varphi)(\xi)\geq {2\eps\over 3}\;\big|\;\xi\in B_\eps\Big) & = & 1\;,
\enqs
for $m$ large enough. From the local uniform convergence of $\varphi_\ell$ to $\varphi$, we get 
\beqs
\P\Big(D NN^{\theta'_{m,\eps}} ~ \in  ~\Cc_\eps;~(NN^{\theta'_{m,\eps}}-\varphi_\ell)(\xi)\geq {\eps}\;\big|\;\xi\in B_\eps\Big) & = & 1\;,
\enqs
for $\ell$ large enough. Moreover, we have from \reff{estim-mol-FL} and \reff{convNN1}
\beqs
\limsup_{m\rightarrow+\infty}\E\Big[|NN^{\theta'_{m,\eps}}(\xi)-F_\Cc[\varphi](\xi)|^2\;\big|\;\xi\in B_\eps\Big] & \leq & 
2\limsup_{m\rightarrow+\infty}\E\Big[|NN^{\theta'_{m,\eps}}(\xi)-\phi_{n_\eps}(\xi)|^2\;\big|\;\xi\in B_\eps\Big] \\& & 
+2\E\Big[|\phi_{n_\eps}(\xi)-F_\Cc[\varphi](\xi)|^2\;\big|\;\xi\in B_\eps\Big]\\
 & \leq & {\eps^2\over 2}\;.
\enqs 
%

\vspace{2mm}

\noindent \textit{Step 2.}
From the definition \reff{defThetastar} of $\theta^*_{m,\eps,\ell}$ we get
\beqs
\E\Big[ \big|(NN^{\theta^*_{m,\eps,\ell}}-\varphi_\ell)(\xi)\big|^2 \;\big|\;\xi\in B_\eps\Big] & \leq & \E\Big[ \big|(NN^{\theta'_{m,\eps}}-\varphi_\ell)(\xi)\big|^2\;\big|\;\xi\in B_\eps \Big] \;.
\enqs 
By sending $\ell$ and $m$ to $\infty$, we get from Step 1
\beq\nonumber
\limsup_{m\rightarrow+\infty}\limsup_{\ell\rightarrow+\infty}\E\Big[ \big|\big(NN^{\theta^*_{m,\eps,\ell}}-\varphi_\ell\big)(\xi)\big|^2 \;\big|\;\xi\in B_\eps\Big] & \leq & \\ \label{estim-ineq-}
 \Big(\E\Big[ \big|\big(F_\Cc[\varphi]-\varphi\big)(\xi)\big|^2\;\big|\;\xi\in B_\eps \Big]^{1\over2} +\sqrt{\eps^2\over2}\Big)^{2}\;.
\enq
Hence, we have 
\beq
\limsup_{\eps\rightarrow0}\limsup_{m\rightarrow+\infty}\limsup_{\ell\rightarrow+\infty}\E\Big[ \big|\big(NN^{\theta^*_{m,\eps,\ell}}-\varphi_\ell\big)(\xi)\big|^2 \;\big|\;\xi\in B_\eps\Big] & \leq & 
\label{estim-ineq-bis}
\E\Big[ \big|\big(F_\Cc[\varphi]-\varphi\big)(\xi)\big|^2
\Big]
\,.\qquad\quad  
\enq
We now define the local facelift operator $F^\eps_{\Cc_\eps}$ by
\beqs
F^\eps_{\Cc_\eps}[\phi](x) & = & \sup_{y\in\R^d\;:\;x+y\in B_\eps} \{\phi(x+y)-\delta_{\Cc_\eps}(y)\}
\enqs
for a locally bounded function $\phi$ and $x\in B_\eps$. We observe that
\beq\label{ineq-facelift-eps}
NN^{\theta^*_{m,\eps,\ell}} & \geq & F^\eps_{\Cc_\eps}[\varphi_\ell]-\eps ~~\mbox{ on }~ B_\eps\;.
\enq
Indeed, from Taylor's formula and since $D NN^{\theta^*_{m,\eps,\ell}}\in \Cc_\eps$ on $B_\eps$ we first have
\beqs
NN^{\theta^*_{m,\eps,\ell}} (x)-\big(NN^{\theta^*_{m,\eps,\ell}} (x+y)-\delta_{\Cc_{\eps}}(y)\big) & = & 
\int_0^1\big( \delta_{\Cc_{\eps}}(y)-D NN^{\theta^*_{m,\eps,\ell}} (x+sy).y\big)ds ~~ \geq ~~ 0
\enqs
for all $x\in B_\eps$ and $y\in\R^d$ such that $x+y\in B_\eps$. Therefore 
\beqs
NN^{\theta^*_{m,\eps,\ell}} & = & F_{\Cc_\eps}^\eps[NN^{\theta^*_{m,\eps,\ell}}]\quad \mbox{  on } ~B_\eps \;.
\enqs 
Since $(NN^{\theta^*_{m,\eps,\ell}}-\varphi_\ell)(\xi)\geq -\eps$ on $B_\eps$ we have
\beq
NN^{\theta^*_{m,\eps,\ell}}  & = &F^\eps_{\Cc_\eps}[ NN^{\theta^*_{m,\eps,\ell}}]
 ~~ \geq ~~ F^\eps_{\Cc_\eps}[\varphi_\ell]-\eps \qquad \mbox{ on }~ B_\eps\;.\label{liminffirst}
\enq
From the uniform convergence of $\varphi_\ell$ to $\varphi$ on compact sets, we have
\beq\label{lim-FL-loc-loc}
(F^\eps_{\Cc_\eps}[\varphi_\ell]-F_{\Cc}[\varphi])(\xi)\mathds{1}_{B_\eps}(\xi) & \xrightarrow[\eps\rightarrow0,\;\ell\rightarrow+\infty]{\P-a.s.} & 0\;.
\enq
Therefore, we get
\beqs
\liminf_{\eps\rightarrow0}\liminf_{m\rightarrow+\infty}\liminf_{\ell\rightarrow+\infty}\big(NN^{\theta^*_{m,\eps,\ell}}-\varphi\big)(\xi) \mathds{1}_{B_\eps}(\xi)& \geq & \big(F_{\Cc}[\varphi]-\varphi\big)(\xi)~~\geq~~0\;,
\enqs
and
\beq\label{estim-ineq-ter}
\liminf_{\eps\rightarrow0}\liminf_{m\rightarrow+\infty}\liminf_{\ell\rightarrow+\infty}\big(NN^{\theta^*_{m,\eps,\ell}}-\varphi\big)^2(\xi)\mathds{1}_{B_\eps}(\xi) & \geq & \big(F_{\Cc}[\varphi]-\varphi\big)^2(\xi)\;.
\enq
From \reff{estim-ineq-bis}, \reff{estim-ineq-ter} and Lemma \ref{lem-convVA} we get
\beqs
\big(NN^{\theta^*_{m,\eps,\ell}}-\varphi\big)^2(\xi)\mathds{1}_{B_\eps}(\xi) & \xrightarrow[\eps\rightarrow 0,m\rightarrow+\infty,\ell\rightarrow+\infty]{\P-p.s.} & \big(F_{\Cc}[\varphi]-\varphi\big)^2(\xi)\;.
\enqs
We deduce from \reff{liminffirst} and \reff{lim-FL-loc-loc} 
\beqs
\big(NN^{\theta^*_{m,\eps,\ell}}-F_{\Cc}[\varphi]\big)(\xi)\mathds{1}_{B_\eps}(\xi) & \xrightarrow[\eps\rightarrow 0,m\rightarrow+\infty,\ell\rightarrow+\infty]{\P-p.s.} & 0\;.
\enqs
We then notice that since $ D NN^{\theta^*_{m,\eps,\ell}}\in \Cc_\eps$ on $B_\eps$ and $\Cc$ is bounded, the family $NN^{\theta^*_{m,\eps,\ell}}$ satisfies a uniform linear growth property for $\eps$ in the neighborhood of $0^+$. Since $F_{\Cc}[\varphi]$ is bounded and $\E[|\xi|^2]<+\infty$, we can apply the dominated convergence Theorem and we get \reff{estim-approx-FL1}
The last result \reff{estim-approx-FL2} is a consequence of \reff{estim-approx-FL1}, the square integrability of $\xi$ and the bound $|F_\Cc[\varphi]|\leq M$.
\end{proof}
\subsection{The approximation scheme}
We fix an initial condition $X_0$ at time $t=0$ for the diffusion and we write $X$ for $X^{0,X_0}$.
We first fix two time grids 
\begin{itemize}
\item a constraint grid $\Rc=\{r_0=0<r_1<\ldots<r_\kappa=T\}$,
\item a family of grids $\pi=\{\pi_k,\;k=0,\ldots,\kappa-1\}$ where $\pi_k$ is a grid of $[r_k,r_{k+1}]$ of the form $\pi_k=\{t_{k,0}=r_k<\ldots<t_{k,n_k}=r_{k+1}\}$\;. We set $|\pi_k|=\max_{i=0,\ldots n_k-1}(t_{k,i+1}-t_{k,i})$.
\end{itemize}
We denote by $X^\pi$ the Euler scheme of $X$ related to the grid $\pi$. It is defined by $X^\pi_0  =  X_0$
and 
\begin{equation*}\left\{\begin{array}{rcl}
X^\pi_{t_{k,i+1}} & = & X^\pi_{t_{k,i}}+b(t_{k,i},X^\pi_{t_{k,i}})\Delta t_{k,i}+\sigma(t_{k,i},X^\pi_{t_{k,i}})\Delta B_{t_{k,i}}\\
X^\pi_{t_{k+1,0}} & = & X^\pi_{t_{k,n_k}}
\end{array}\right.\end{equation*}
with $\Delta t_{k,i} ~ = ~t_{k,i+1}-t_{k,i}$ and $\Delta B_{t_{k,i}} ~=~B_{t_{k,i+1}} - B_{t_{k,i}}$
for $k=0,\ldots, \kappa-1$ and $i=0,\ldots,n_k-1$. We then introduce the function $F:~[0,T]\times\R^d\times \R\times \R^d\times [0,T]\times \R^d\rightarrow\R$ defined by
\beqs
F(t,x,y,z,h,\Delta)  & :=  & y - f(t,x,y,z)h + z.\Delta
\enqs
for $(t,x,y,z,h,\Delta)\in[0,T]\times\R^d\times \R\times \R^d\times [0,T]\times \R^d$.
We fix two multi-parameters $\eps=(\eps_0,\eps_1,\ldots,\eps_\kappa)$ and $m=(m^1_0,m^2_0,m^3_0,m^1_1,m^2_1,m^3_1,\ldots$ $\ldots,m^1_{\kappa-1},m^2_{\kappa-1},m^3_{\kappa-1},m^1_{\kappa})$ and two positive constants $M$ and $L$. We define $\{\mathcal{V}^{\Rc,\pi,\eps,m}_{k,i}\}^{0\leq k\leq \kappa-1}_{ 0\leq i\leq n_k}$ and $( \tilde{ \mathcal{V}}^{\Rc,\pi,\eps,m}_{k,i})^{0\leq k\leq \kappa-1}_{ 0\leq i\leq n_k}$ by the following algorithm.\\
\begin{algorithm2e}[H]
\DontPrintSemicolon 
\SetAlgoLined 
\vspace{1mm}
\vspace{0.5mm}
{\small \beqs
\mathcal{V}^{\Rc,\pi,\eps,m}_{\kappa,0} ~ = ~\tilde {\mathcal{V}}^{\Rc,\pi,\eps,m}_{\kappa,0} & = & NN^{\theta^*_{\kappa,0}}\wedge(-M)\vee M
\enqs
}
where
{\small
\beqs
 \theta^*_{\kappa,0} & \in & \textrm{arg}\min_{ } \E\Big[\Big| 
NN^{\theta}(X^\pi_{T})-g(X^\pi_{T})\Big|^2\;\big|\;X^\pi_{T}\in B_{\eps_\kappa}\Big]\\
 & & ~~ \theta\in \Theta_{m^1_\kappa} \text{ s.t. }  \P\Big(D  NN^\theta(X^\pi_T)\in \Cc_{\eps_\kappa} ; 
NN^{\theta}(X^\pi_T)\geq g(X^\pi_T)-{\eps_\kappa}\;\big|\;X^\pi_T\in B_{\eps_\kappa}\Big)=1\;.
\enqs
}
\For{ $k=\kappa-1,\ldots,0$}
{
\beqs
\mathcal{V}^{\Rc,\pi,\eps,m}_{k,n_{k}} ~ = ~ \mathcal{V}^{\Rc,\pi,\eps,m}_{k+1,0} & \mbox{ and } & \tilde {\mathcal{V}}^{\Rc,\pi,\eps,m}_{k,n_{k}} ~ = ~ \tilde {\mathcal{V}}^{\Rc,\pi,\eps,m}_{k+1,0}\;.
\enqs
  {\small  
\For{$i=n_{k}-1,\ldots,1$}
    {
  
\beqs
\tilde{\mathcal{V}}^{\Rc,\pi,\eps,m}_{k,i}~~=~~
\mathcal{V}^{\Rc,\pi,\eps,m}_{k,i} & = & NN^{\theta^*_{k,i}}
\enqs
where
\beqs
(\theta^*_{k,i},\hat \theta^*_{k,i}) & \in & \textrm{arg}\hspace{-8mm}\min_{ (\theta,\hat\theta)\in \Theta_{m^3_{k}}\times \Theta_{m^3_{k}}^d}
\E\Big[\Big| 
NN^{\theta^*_{{k,i+1}}}(X^\pi_{t_{k,i+1}})\\
 & & \quad\quad -F\big(t_{k,i},X^\pi_{t_{k,i}},NN^\theta(X^\pi_{t_{k,i}}),  NN^{\hat \theta}(X^\pi_{t_{k,i}}),\Delta t_{k,i}, \Delta B_{t_{k,i}} \big) \Big|^2\Big] \;.
\enqs
}
\beqs
\mathcal{V}^{\Rc,\pi,\eps,m}_{k,0} ~ = ~ NN^{\theta^*_{k,0}}\wedge(-M)\vee M & \mbox{ and } & \tilde {\mathcal{V}}^{\Rc,\pi,\eps,m}_{k,0} ~ = ~ NN^{\tilde \theta^*_{k,0}}
\enqs
where
{\footnotesize
\beqs
\theta^*_{k,0} & \in & \textrm{arg}\min_{ } \E\Big[\Big| 
NN^{\theta}(X^\pi_{t_{k,0}})-NN^{\check \theta^*_{{k,0}}}(X^\pi_{t_{k,0}})\wedge(-M)\vee M\Big|^2\;\big|\;X^\pi_{t_{k,0}}\in B_{\eps_{k}}\Big]\\
 & & ~~ \theta\in\Theta_{m^1_{k}} \text{ s.t. }  \P\Big(D  NN^\theta(X^\pi_{t_{k,0}})\in \Cc_{\eps_k} ; \\
  & & ~\qquad\qquad 
NN^{\theta}(X^\pi_{t_{k,0}})\geq NN^{\check \theta^*_{{k,0}}}(X^\pi_{t_{k,0}})\wedge(-M)\vee M-{\eps_k}\;\big|\;X^\pi_{t_{k,0}}\in B_{\eps_{k}}\Big)=1\;,\\
 \check \theta^*_{k,0} & \in & \textrm{arg}\min_{ } \E\Big[\Big| 
NN^{\theta}(X^\pi_{t_{k,0}})\wedge(-M)\vee M-NN^{\tilde \theta^*_{{k,0}}}(X^\pi_{t_{k,0}})\Big|^2\;\big|\;X^\pi_{t_{k,0}}\in B_{\eps_{k}}\Big]\\
& &  \theta=(\lambda_i,\alpha_i)_{1\leq i\leq m^2_k}\in \Theta_{m^2_k} \text{ s.t. }
  |\sum_{i=1}^{m^2_k}\lambda_i\alpha_i|\leq {L+1\over |\rho'(0)|}\\
\tilde\theta^*_{k,0} & \in & \textrm{arg}\hspace{-8mm}\min_{  (\theta,\hat\theta)\in \Theta_{m^3_{k}}\times \Theta_{m^3_{k}}^d}\hspace{-4mm} \E\Big[\Big|
NN^{\theta^*_{{k,1}}}(X^\pi_{t_{k,1}})\\
 & & \quad\quad
 -F\big(t_{k,0},X^\pi_{t_{k,0}},NN^\theta(X^\pi_{t_{k,0}}),  NN^{\hat \theta}(X^\pi_{t_{k,0}}),\Delta t_{k,0}, \Delta B_{t_{k,0}} \big) \Big|^2\Big] \;.
 \enqs
}
}
}
\caption{Global approximation scheme. \label{algoGlobScheme}}
\end{algorithm2e}

\vspace{5mm}


We choose the constants $M$ and $L$ such that the functions $v^\Rc$ are $L$-Lipschitz continuous and bounded by $M$. We recall that such constants exist from Proposition \ref{reg-v-pi}.   

The sequences $\{\mathcal{V}^{\Rc,\pi,\eps,m}_{k,i}(X^\pi_{t_{k,i}})\}^{0\leq k\leq \kappa-1}_{ 0\leq i\leq n_k}$
 and $\{ \tilde{ \mathcal{V}}^{\Rc,\pi,\eps,m}_{k,i}(X^\pi_{t_{k,i}})\}^{0\leq k\leq \kappa-1}_{ 0\leq i\leq n_k-1}$
  play the role of approximations for $\{Y^\Rc_{t_{k,i}}\}^{0\leq k\leq \kappa-1}_{ 0\leq i\leq n_k}$
   and $\{Y^\Rc_{t_{k,i}}\}^{0\leq k\leq \kappa-1}_{ 0\leq i\leq n_k}$
    respectively. We then also define the approximation $\{\bar{\hat{ \Zc}}^{\Rc,\pi}_{k,i}\}^{ 0 \leq k\leq \kappa-1}_{0\leq i\leq n_k-1}$ of the process $Z^\Rc$ by
\beqs
\bar{ \hat{ \Zc}}^{\Rc,\pi}_{k,i} & = &  NN^{\hat\theta_{k,i}}(X^\pi_{t_{k,i}})\;,
\enqs
for $k=0,\ldots,\kappa-1$ and $i=0,\ldots,n_k-1$.

\subsection{Convergence of the approximation scheme}
To study the behavior of the approximation Algorithm \ref{algoGlobScheme}, 
we make the additional standing assumptions on the drift $b$, the diffusion coefficient $\sigma$ and the driver $f$. 

\vspace{2mm}

\noindent \textbf{(H$b,\sigma$)'}  There exists a constant $L_{b,\sigma} > 0$ such that 
\beqs
|b(t,x)-b(t',x')|+|\sigma(t,x)-\sigma(t',x')| & \leq & L_{b,\sigma} \Big(|t -t'|^{1\over 2} +|x -x'|\Big)
\enqs
for all $(t,x)$ and $(t',x')\in[0,T]\times\R^d $.

\vspace{2mm}

\noindent \textbf{(H$f$)'}  There exists a constant $L_f > 0$ such that 
\beqs
|f(t,x,y,z)-f(t',x',y',z')| & \leq & L_f \Big(|t -t'|^{1\over 2} +|x -x'|+|y -y'|+|z -z'|\Big)
\enqs
for all $(t,x,y,z)$ and $(t',x',y',z')\in[0,T]\times\R^d \times\R\times\R^d$.

\vspace{2mm}

We next define the error $\textrm{Err}^{\pi,\Rc}$ related to the grids $\pi$ and $\Rc$
\beqs
\textrm{Err}^{\pi,\Rc}_{\eps,m} & = & \max_{k=0,\ldots,\kappa-1}\max_{i=1,\ldots,n_k}\E\Big[ \big|Y^\Rc_{t_{k,i}}-\mathcal{V}^{\Rc,\pi,\eps,m}_{k,i}(X^\pi_{t_{k,i}})\big|^2 \Big]\\
 & & + \max_{k=0,\ldots,\kappa-1}\max_{i=0,\ldots,n_k-1}\E\Big[ \big|\tilde Y^\Rc_{t_{k,i}}-\tilde{\mathcal{V}}^{\Rc,\pi,\eps,m}_{k,i}(X^\pi_{t_{k,i}})\big|^2 \Big]\\
  &  & +\E\Big[ \sum_{i=0}^{n-1}\int_{t_{k,i}}^{t_{k,i+1}}\big|Z^\Rc_{t}-\bar{ \hat{ \Zc}}^{\Rc,\pi}_{k,i}\big|^2 dt \Big] \;.
\enqs

We then have the following convergence result.
\begin{Theorem}\label{Thm-conv-scheme} We have the following convergence
\beqs
\lim_{n_{0}\rightarrow+\infty}\lim_{m_{0}^3\rightarrow+\infty}\lim_{\eps_1\rightarrow0}\lim_{m^1_{1}\rightarrow+\infty}\lim_{m^2_{1}\rightarrow+\infty}\lim_{n_{1}\rightarrow+\infty}\lim_{m_{1}^3\rightarrow+\infty} \ldots\qquad\qquad\qquad\qquad& & \\\cdots\lim_{\eps_{\kappa-1}\rightarrow0}\lim_{m^1_{\kappa-1}\rightarrow+\infty}\lim_{m^2_{\kappa-1}\rightarrow+\infty}\lim_{n_{\kappa-1}\rightarrow+\infty}\lim_{m^3_{\kappa-1}\rightarrow+\infty}\lim_{\eps_\kappa\rightarrow0}\lim_{m^1_{\kappa}\rightarrow+\infty}\textrm{Err}_{\eps,m}^{\pi,\Rc} & = & 0\;.
\enqs
\end{Theorem}
To prove Theorem \ref{Thm-conv-scheme}, we need the two following lemmata.
\begin{Lemma} \label{lem-convL2cvLUC}
Let $\varphi$ and $\varphi_\ell$, $\ell\geq1$, be   functions from $\R^d$ to $\R$. Suppose there exists constants $L$ and $M$ such that $\varphi$ and $\varphi_\ell$, $\ell\geq1$, are $L$-Lipschitz continuous and bounded by $M$. Let $\xi$ be a random variable such that $\textrm{Supp}(\P_\xi)=\R^d$ and suppose  also that 
\beq\label{cond-conv-phiell-phi}
\E\big[\big|\varphi_\ell(\xi)-\varphi(\xi)|^2\big] & \xrightarrow[\ell\rightarrow+\infty]{} & 0 \;.
\enq
Then $\varphi_\ell$ converges uniformly to $\varphi$ on compact subsets of $\R^d$.
\end{Lemma}
\begin{proof}
From Ascoli Theorem the sequence $(\varphi_\ell)_{\ell \geq 1}$ is compact for the convergence on compact subsets of $\R^d$. Let $\tilde \varphi$ be an adherence value. Then, up to a subsequence
\beqs
\sup_{K}|\varphi_\ell-\tilde \varphi| & \xrightarrow[\ell\rightarrow+\infty]{} & 0
\enqs
for any compact subset $K$ of $\R^d$. From \reff{cond-conv-phiell-phi}, we deduce that 
$\tilde \varphi(\xi)  =  \varphi(\xi)$ $\P$-a.s.
and since $\textrm{Supp}(\P_\xi)=\R^d$ we get $\tilde \varphi=\varphi$ on $\R^d$. 
\end{proof}
The next results shows that for the approximation of a bounded and Lipschitz continuous function, we can restrict the neural network weights  to a given bound. 
\begin{Lemma}\label{lem-NN-bdd}
Let $\varphi$ and $\varphi_\ell$, $\ell\geq1$, be   functions from $\R^d$ to $\R$ and 
$\xi$ be a random variable satisfying conditions of Lemma \ref{lem-convL2cvLUC}. 
Suppose the activation function $\rho$ is differentiable with $\rho'(0)\neq 0$.
Define the sequence $(\theta^*_{m,\ell})_{m,\ell\geq1}$ by
\beqs
\theta^*_{m,\ell} & \in &  \text{arg}\min_{}\E\big[\big|NN^{\theta}(\xi)\vee(-M)\wedge  M-\varphi_\ell(\xi)|^2\big]\\
 &  & \theta=(\lambda_i,\alpha_i)_{1\leq i\leq m}\in \Theta_m \text{ s.t. }
  |\sum_{i=1}^m\lambda_i\alpha_i|\leq {L+1\over |\rho'(0)|} \;.
\enqs
 Then 
 \beq\label{conv-phiell-phi2}
\lim_{m\rightarrow+\infty}\lim_{\ell\rightarrow+\infty}\E\big[\big|NN^{\theta_{m,\ell}^*}(\xi)\vee(-M)\wedge  M-\varphi(\xi)|^2\big] & = & 0 \;.
\enq
\end{Lemma}
\begin{proof} Using a mollification argument, we can assume w.l.o.g. that $\varphi\in C^1(\R^d,\R)$. From Theorem \ref{ThmHornikdiff}, we can find a sequence $(\theta_m)_{m\geq1}$ such that $\theta_m\in \Theta_m$ for $m\geq1$ and $(NN^{\theta_m},DNN^{\theta_m})_{m\geq1}$ converges uniformly to $(\varphi, D\varphi)$ on compact sets. We therefore get for $m$ large enough
\beqs
 |\sum_{i=1}^m\lambda^m_i\alpha_i^m| & \leq  &{L+1\over |\rho'(0)|}
\enqs
where $\theta_m = (\lambda^m_i, \alpha_i^m)_{1\leq i\leq m}$.
From the definition of $\theta^*_{m,\ell}$ we have
\beqs
\E\big[\big|NN^{\theta^*_{m,\ell}}(\xi)\vee(-M)\wedge  M-\varphi(\xi)|^2\big]  & \leq & 2\E\big[\big|NN^{\theta^*_{m,\ell}}(\xi)\vee(-M)\wedge  M-\varphi_\ell(\xi)|^2\big]\\
  & & +2\E\big[\big|\varphi_\ell(\xi)-\varphi(\xi)|^2\big] \\
 & \leq & 2\E\big[\big|NN^{\theta_{m}}(\xi)\vee(-M)\wedge  M-\varphi_\ell(\xi)|^2\big]\\
  & & +2\E\big[\big|\varphi_\ell(\xi)-\varphi(\xi)|^2\big] 
\enqs
which converges to zero as $\ell$ and $m$ goes to $\infty$.  
\end{proof}
\begin{Remark}
 If we suppose the derivative of the activation function $\rho$ is bounded  by a constant $C$ then, the condition $|\sum_{i=1}^m\lambda_i\alpha_i|\leq {L+1\over |\rho'(0)|}$ restricts to neural networks that are $C{L+1\over |\rho'(0)|}$ Lipschitz continuous.
\end{Remark}
\begin{proof}[Proof of Theorem \ref{Thm-conv-scheme}.]
We recall that for $(t,x)\in[0,T]\times\R^d$, $(Y^{\Rc,t,x}, Z^{\Rc,t,x})$ is defined by \reff{EDSRDC1}-\reff{EDSRDC2}-\reff{EDSRDC3}. 
%
From Proposition \ref{prop-reg-v} and classical estimates on Euler scheme we have
\beq\nonumber
 \max_{k=0,\ldots,\kappa-1}\max_{i=1,\ldots,n_k}\E\Big[ \big|Y^{\Rc}_{t_{k,i}}-Y^{\Rc,t_{k,0},X^\pi_{t_{k,0}}}_{t_{k,i}}\big|^2 \Big] & & \\
 + \max_{k=0,\ldots,\kappa-1}\max_{i=0,\ldots,n_k-1}\E\Big[ \big|\tilde Y^{\Rc}_{t_{k,i}}-\tilde Y^{\Rc,t_{k,0},X^\pi_{t_{k,0}}}_{t_{k,i}}\big|^2 \Big]   & & \label{proc-intermed}\\
+ \sum_{k=0}^{\kappa-1}\E\Big[ \sum_{i=0}^{n_k-1}\int_{t_{k,i}}^{t_{k,i+1}}\big|Z^\Rc_{t}-{{\Zc}}^{\Rc,t_{k,0},X^\pi_{t_{k,0}}}_{t_{k,i}}\big|^2 dt \Big]  & \longrightarrow & 0 \;,\nonumber
\enq
as $\max_{0\leq k\leq \kappa-1}|\pi_k|\rightarrow0$.
From Proposition \ref{prop-comp-facelift} we have
\beqs
\lim_{\eps_\kappa\rightarrow0}\lim_{m^1_{\kappa}\rightarrow+\infty}\E\Big[ |(\mathcal{V}^{\Rc,\pi,\eps,m}_{\kappa-1,n_{\kappa-1}}-F_\Cc[g])\mathds{1}_{B_{\eps_\kappa}}|^2(X^\pi_{\kappa-1,n_{\kappa-1}})\Big] & = & 0\;.
\enqs
Since $F_\Cc[g]$ is Lipschitz continuous, we get from  Theorem 4.1 in \cite{HPW19} and Corollary 2.2  in \cite{HSW89}
\beqs
\lim_{n_{\kappa-1}\rightarrow+\infty}\lim_{m^3_{\kappa-1}\rightarrow+\infty}\lim_{\eps_\kappa\rightarrow0}\lim_{m^1_{\kappa}\rightarrow+\infty} \max_{i=1,\ldots,n_{\kappa-1}}\E\Big[ \big|Y^{\Rc,t_{\kappa-1,0},X^\pi_{t_{\kappa-1,0}}}_{t_{\kappa-1,i}}-\mathcal{V}^{\Rc,\pi,\eps,m}_{\kappa-1,i}(X^\pi_{t_{\kappa-1,i}})\big|^2 \Big] & & \\
  + \max_{i=0,\ldots,n_{\kappa-1}-1}\E\Big[ \big|\tilde Y^{\Rc,t_{\kappa-1,0},X^\pi_{t_{\kappa-1,0}}}_{t_{\kappa-1,i}}-\tilde{\mathcal{V}}^{\Rc,\pi,\eps,m}_{\kappa-1,i}(X^\pi_{t_{k,i}})\big|^2 \Big]  &  & \\
+\E\Big[ \sum_{i=0}^{n_{\kappa-1}-1}\int_{t_{\kappa-1,i}}^{t_{\kappa-1,i+1}}\big|Z^{\Rc,t_{\kappa-1,0},X^\pi_{t_{\kappa-1,0}}}_{t}-\bar{ \hat{ \Zc}}^{\Rc,\pi}_{{\kappa-1},i}\big|^2 dt \Big] & = & 0\;.
\enqs 
From Proposition \ref{prop-comp-facelift}, Lemmata \ref{lem-convL2cvLUC} and \ref{lem-NN-bdd} and the previous convergence, we get
\beqs
\lim_{\eps_{\kappa-1}\rightarrow0}\lim_{m^1_{\kappa-1}\rightarrow+\infty}\lim_{m^2_{\kappa-1}\rightarrow+\infty}\lim_{n_{\kappa-1}\rightarrow+\infty} \qquad\qquad\qquad\qquad\qquad\qquad& & \\\lim_{m^3_{\kappa-1}\rightarrow+\infty}\lim_{\eps_\kappa\rightarrow0}\lim_{m^1_{\kappa}\rightarrow+\infty}\E\Big[ \big|Y^{\Rc,t_{\kappa-1,0},X^\pi_{t_{\kappa-1,0}}}_{t_{\kappa-1,0}}-\mathcal{V}^{\Rc,\pi,\eps,m}_{\kappa-1,0}(X^\pi_{t_{\kappa-1,0}})\big|^2 \Big] & = & 0\;.
\enqs
Repeating this argument for each $k=\kappa-2,\ldots,0$, and using \reff{proc-intermed}, we get the result.
\end{proof}
We end this section by a convergence result for the constrained solution.
Take $(\Rc^\ell)_{\ell\geq 1}$ a nondecreasing sequence  such that
\beqs
|\mathcal R^\ell|~~:=~~\max_{1\leq k\leq \kappa_\ell}(r^\ell_{k}-r^\ell_{k-1}) & \xrightarrow[n\rightarrow+\infty]{} & 0\;,
\enqs and  define
\beqs
\widehat{\textrm{Err}}^{\pi,\ell}_{\eps,m} & = & \max_{i=0,\ldots,n-1}\E\Big[ \big|Y_{t_i}-\mathcal{V}^{\Rc^\ell,\pi}_i(X^\pi_{t_i})\big|^2 \Big]+\E\Big[ \sum_{i=0}^{n-1}\int_{t_i}^{t_{i+1}}\big|Z_{t}-\bar{ \hat{ \Zc}}^{\Rc^\ell,\pi}_i\big|^2 dt \Big] \;.
\enqs
From  Corollary \ref{cor-cv-disc-const} and Theorem \ref{Thm-conv-scheme} we obtain the following result.
\begin{Corollary} We have the following convergence
\beqs
\lim_{\ell\rightarrow+\infty}\lim_{n_{0}\rightarrow+\infty}\lim_{m_{0}^3\rightarrow+\infty}\lim_{\eps_1\rightarrow0}\lim_{m^1_{1}\rightarrow+\infty}\lim_{m^2_{1}\rightarrow+\infty}\lim_{n_{1}\rightarrow+\infty}\lim_{m_{1}^3\rightarrow+\infty} \ldots\qquad\qquad\qquad\qquad& & \\\cdots\lim_{\eps_{\kappa-1}\rightarrow0}\lim_{m^1_{\kappa-1}\rightarrow+\infty}\lim_{m^2_{\kappa-1}\rightarrow+\infty}\lim_{n_{\kappa-1}\rightarrow+\infty}\lim_{m^3_{\kappa-1}\rightarrow+\infty}\lim_{\eps_\kappa\rightarrow0}\lim_{m^1_{\kappa}\rightarrow+\infty}\widehat{\textrm{Err}}^{\pi,\ell}_{\eps,m} & = & 0\;.
\enqs
\end{Corollary}

\section{Numerical results}\label{sec5}
\subsection{Neural network approximation}
In the sequel we first show that we can approximate the facelift easily with 
neural networks. In a second part we test the global algorithm evaluating the BSDE with constraints.
\subsubsection{Testing the facelift approximation of a function $\varphi$}
Testing many penalizing function, it turns out that the use of simple Relu function is the best way to simply penalize the constraints introducing a second small parameters $\epsilon_1$. This function prevents the problem of vanishing gradient that may appear using some regularization of some heaviside function for example. \\
We propose to  use a $L_1$ norm  on the distance to the target and the penalty terms giving coefficients of the neural network satisfying
\beq
\label{eq:obj1}
\theta_{m,\eps}^* & \in & \textrm{arg}\min_{\theta\in \Theta_m}\E\Big[  |NN^\theta-\varphi|(\xi) +   \min_{x \in \Cc} \frac{|| D NN ^\theta(\xi)- x||_1}{\epsilon_1} +  \nonumber\\
& &    \frac{ \big( (\varphi -NN^\theta)(\xi) \big)^{+}}{\epsilon_1} \Big]
\enq
where  $\xi$ is an uniform r.v. in $B_\eps$.\\
\begin{Remark}
The  use of a $L_2$ norm for the distance to the true function or/and the different constraints does not give results as good as with the objective function above.
\end{Remark}
Using a neutral network, we have no certainty to get the facelift of a function $\varphi$. The problem is not convex
and we face a dilemma:
\begin{itemize}
    \item either we use a rather high penality coefficient  $\epsilon_2$ and may not satisfy the constraints,
    \item either we set a very small $\epsilon_2$ and the distance between the estimated facelift and the function is  only seen as some noise by the gradient descent.
\end{itemize}
As we want to use a rather small $\epsilon_2$ parameter, we will get
solutions above the real facelift.
We  then propose to use the iterative algorithm  \ref{algo1} that successively approximates the facelift by above.\\
\begin{algorithm2e}[H]
\DontPrintSemicolon 
\SetAlgoLined 
\vspace{1mm}
\KwIn{Function to facelift $\varphi$}

\vspace{0.5mm}
\beqs
\theta_{m,\eps}^{*,0} & \in & \textrm{arg}\min_{\theta\in \Theta_m}\E\Big[  |NN^\theta-\varphi|(\xi) +   \min_{x \in \Cc} \frac{|| D NN ^\theta(\xi)- x||_1}{\epsilon_1} + \\
& &    \frac{ \big( (\varphi -NN^\theta)(\xi) \big)^{+}}{\epsilon_1} \Big]
\enqs

\For{$k=1,...,K$}
{
\beqs
\theta_{m,\eps}^{*,k} & \in & \textrm{arg}\min_{\theta\in \Theta_m}\E\Big[  |NN^\theta-\varphi|(\xi) +   \min_{x \in \Cc} \frac{|| D NN ^\theta(\xi)- x||_1}{\epsilon_1} + \\
& &    \frac{ \big( (\varphi -NN^\theta)(\xi) \big)^{+}}{\epsilon_1}  + \frac{ \big( (NN^{\theta_{m,\eps}^{*,k-1}}-\varphi)(\xi) \big)^{+}}{\epsilon_1}   \Big]
\enqs
}
\KwOut{$\theta_{m,\eps}^{*,K}$}
\vspace{1mm}

\caption{Iterative algorithm for facelift calculation of a function $\varphi$. \label{algo1}}

\end{algorithm2e}

We test three activation functions ReLU, tanh and ELU with the bounded set  $$\Cc = \{  x \in \R / || x|| \le \hat d \},$$ for different values of $\hat d$.\\
ELU is the less effective while ReLU gives results slightly better than tanh. In the sequel ReLU is taken for numerical results. As for the number of hidden layers, one layer appears to be insufficient and 3 does not bring any improvement comparing to two hidden layers.\\
We have to take at least $100$ neurons per layer to get very good  results. In the sequel  we take $200$ neurons.\\
In the numerical results we take mini batch of size $1000$ with the Adam optimizer \cite{kingma2014adam} using a learning rate equal to $0.001$. We stop the algorithm after 100000 iterations and every hundred iterations we do a more accurate estimation of the loss with $10000$ particles keeping the best network obtained during iterations.\\

We test the algorithm on a fixed convex set depending on the test case.
\paragraph{First case}
\label{sec:flCase1}
For the second test case we use the payoff  of a butterfly function
$$ \varphi(x)= (x - 0.8)^{+} -2 (x-1)^{+} + (x-1.2)^{+} .$$
The facelift function is peacewise linear given  for $\hat d  \le 1$ by 
$$\varphi^A_{ \hat d}(x) = (1- \hat d |x-1|)^{+}. $$
\begin{figure}[H]
\begin{minipage}[b]{0.32\linewidth}
  \centering
 \includegraphics[width=\textwidth]{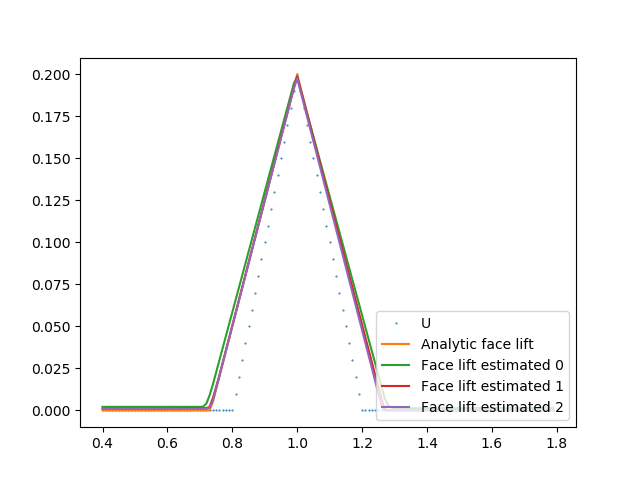}
 \caption*{$\epsilon_2 =\frac{1}{200}$}
 \end{minipage}
 \centering
 \begin{minipage}[b]{0.32\linewidth}
  \centering
 \includegraphics[width=\textwidth]{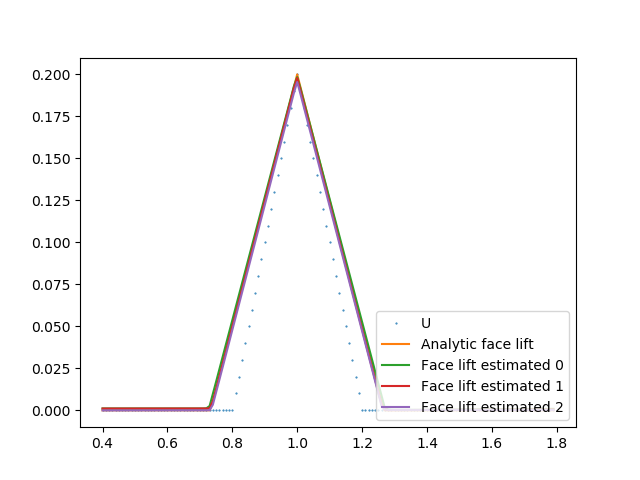}
 \caption*{$\epsilon_2 =\frac{1}{50} $}
 \end{minipage}
 \begin{minipage}[b]{0.32\linewidth}
  \centering
 \includegraphics[width=\textwidth]{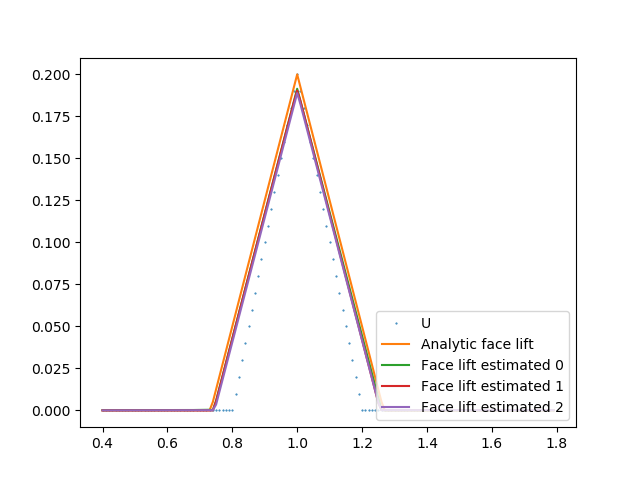}
 \caption*{$\epsilon_2 =\frac{1}{10} $}
 \end{minipage}
 \caption{facelift approximation for different values of $\epsilon_2$ for the first case with $\hat d =0.75$, $K=2$. On each figure, we give the results obtained for each iteration of the algorithm. \label{fig:facelift1_1}}
\end{figure}
\begin{figure}[H]
\begin{minipage}[b]{0.32\linewidth}
  \centering
 \includegraphics[width=\textwidth]{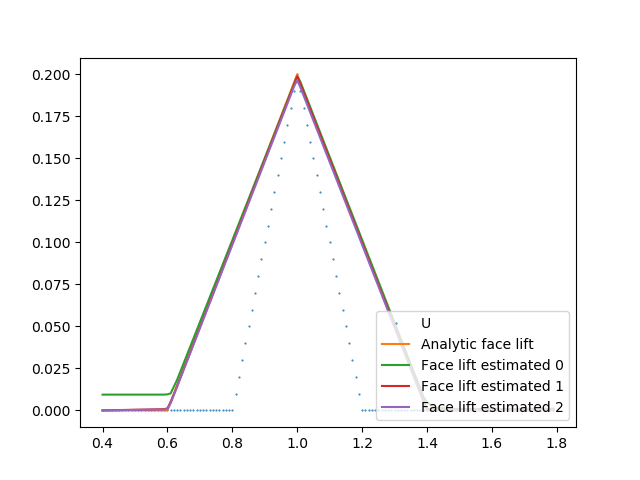}
 \caption*{$\epsilon_2 =\frac{1}{200}$}
 \end{minipage}
 \centering
 \begin{minipage}[b]{0.32\linewidth}
  \centering
 \includegraphics[width=\textwidth]{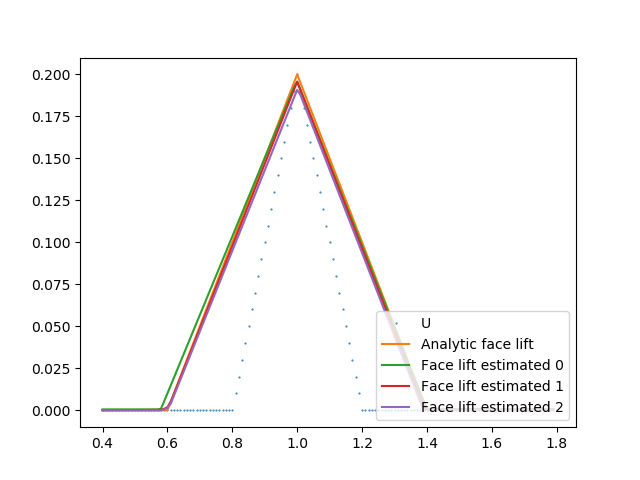}
 \caption*{$\epsilon_2 =\frac{1}{50} $}
 \end{minipage}
 \begin{minipage}[b]{0.32\linewidth}
  \centering
 \includegraphics[width=\textwidth]{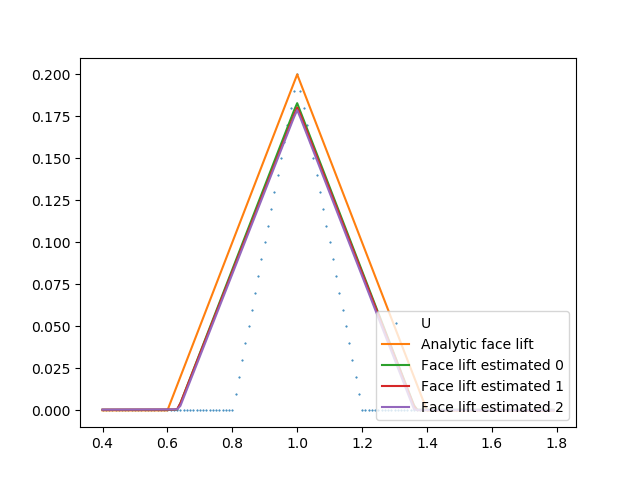}
 \caption*{$\epsilon_2 =\frac{1}{10} $}
 \end{minipage}
 \caption{facelift approximation for different values of $\epsilon_2$ for the first case with $\hat d =0.5$, $K=2$. On each figure, we give the results obtained for each iteration of the algorithm. \label{fig:facelift1_2}}
\end{figure}
On Figures \ref{fig:facelift1_1},\ref{fig:facelift1_2}, we give the facelift obtained for different values of $\epsilon_2$ and $\hat d$.
For a small constraint ($\hat d=7.5$), the facelift is calculated very well for all penalty even with one iteration of the algorithm meaning that a  simple resolution of \eqref{eq:obj1} is sufficient enough.
For a smaller value $\hat d$ a quite high penalty value is necessary to get a good result with at least two iterations of the algorithm.
\paragraph{Second case}
\label{sec:flCase2}
We want to calculate the facelift of the function
\beq
\varphi(x)=  4 [(x - 0.8)^{+} - (x-1)^{+}] +  (x-1.2)^{+}
\label{eq:varphi2}
\enq
on set $[0.6, 1.4]$. The facelift function is obviously piecewise linear and given for $\hat d  \le 4$ by
\begin{equation}
\varphi^A_{\hat d}(x)=  \left \{ 
\begin{array}{c}
4 [(x - 0.8)^{+} - (x_i-1)^{+}] +  (x_i-1.2)^{+},  \quad  x  \ge 1 \\
(0.8 - \hat d |x-1|)^{+}, \quad  x < 1 .
\end{array}
\right.
\label{eq:varphiA1D}
\end{equation}
\begin{figure}[H]
\begin{minipage}[b]{0.32\linewidth}
  \centering
 \includegraphics[width=\textwidth]{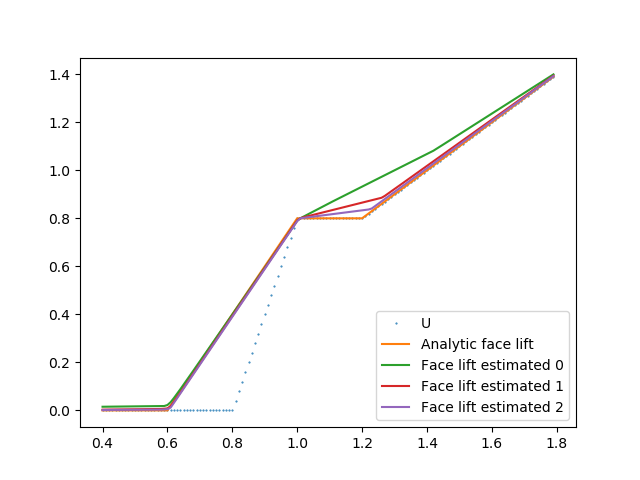}
 \caption*{$\epsilon_2 =\frac{1}{200}$}
 \end{minipage}
 \centering
 \begin{minipage}[b]{0.32\linewidth}
  \centering
 \includegraphics[width=\textwidth]{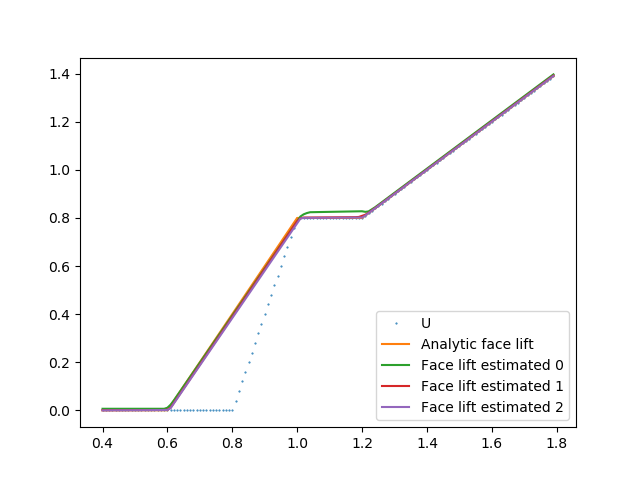}
 \caption*{$\epsilon_2 =\frac{1}{50} $}
 \end{minipage}
 \begin{minipage}[b]{0.32\linewidth}
  \centering
 \includegraphics[width=\textwidth]{OpXWdisk20Eps2_5relulayer2Neuro200.png}
 \caption*{$\epsilon_2 =\frac{1}{20} $}
 \end{minipage}
 \caption{facelift approximation for different values of $\epsilon_2$ for the first case with $\hat d =2$, $K=2$. \label{fig:facelift2_1}}
\end{figure}
\begin{figure}[H]
\begin{minipage}[b]{0.32\linewidth}
  \centering
 \includegraphics[width=\textwidth]{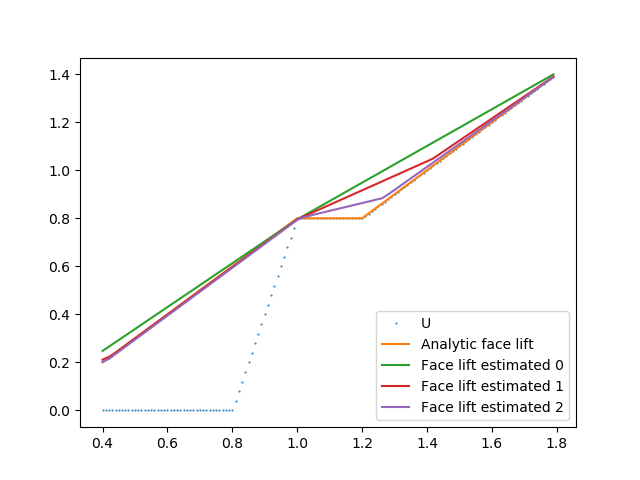}
 \caption*{$\epsilon_2 =\frac{1}{200}$}
 \end{minipage}
 \centering
 \begin{minipage}[b]{0.32\linewidth}
  \centering
 \includegraphics[width=\textwidth]{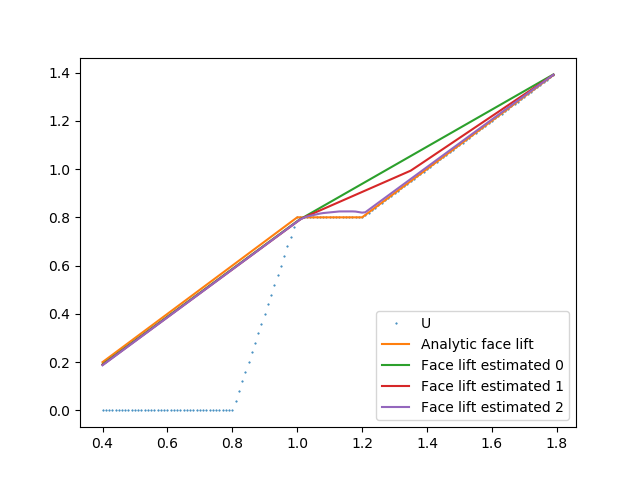}
 \caption*{$\epsilon_2 =\frac{1}{50} $}
 \end{minipage}
 \begin{minipage}[b]{0.32\linewidth}
  \centering
 \includegraphics[width=\textwidth]{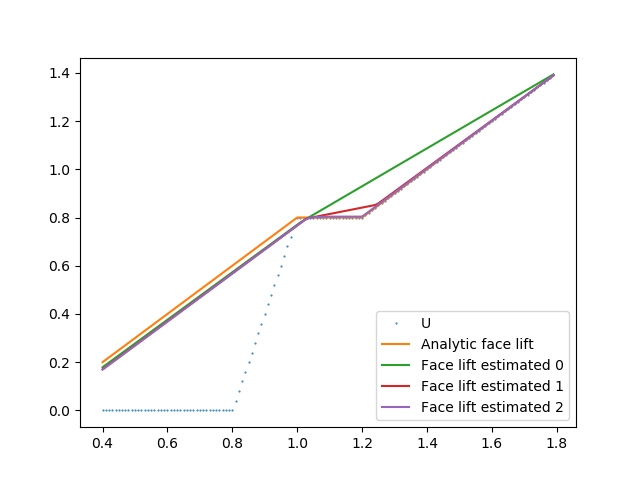}
 \caption*{$\epsilon_2 =\frac{1}{20} $}
 \end{minipage}
 \caption{facelift approximation for different values of $\epsilon_2$ for the first case with $\hat d =1$, $K=2$. \label{fig:facelift2_2}}
\end{figure}
On this test case, at least 3 iterations of the algorithm are necessary to reach a good accuracy. As before, since the constraint is higher, the algorithm faces difficulty to reach a very good accuracy.

\paragraph{Third case}
\label{sec:flCase3}
For this third  case, we take
$$ \varphi(x)= \log( 1+e^x) + 4 \frac{\sin(2x)}{1+5x^2}. $$
On Figure \ref{fig:facelift3}, we give the function value obtained with different values of $\epsilon_2$ using 3 iterations of the algorithm   for different size $\hat d$ . 
\begin{figure}[H]
\begin{minipage}[b]{0.49\linewidth}
  \centering
 \includegraphics[width=\textwidth]{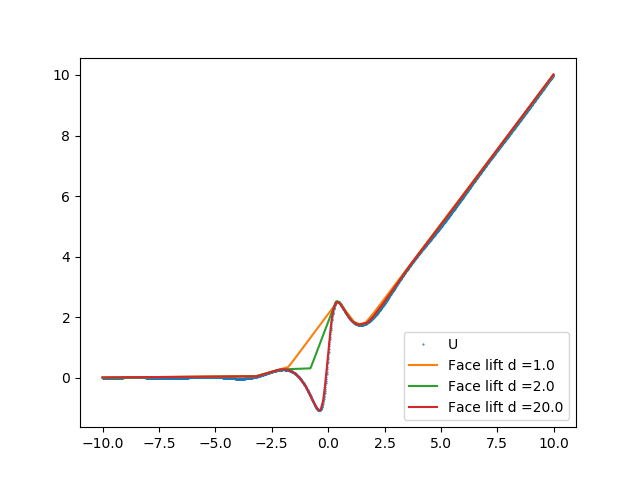}
 \caption*{$\epsilon_2 =\frac{1}{100}$}
 \end{minipage}
 \centering
 \begin{minipage}[b]{0.49\linewidth}
  \centering
 \includegraphics[width=\textwidth]{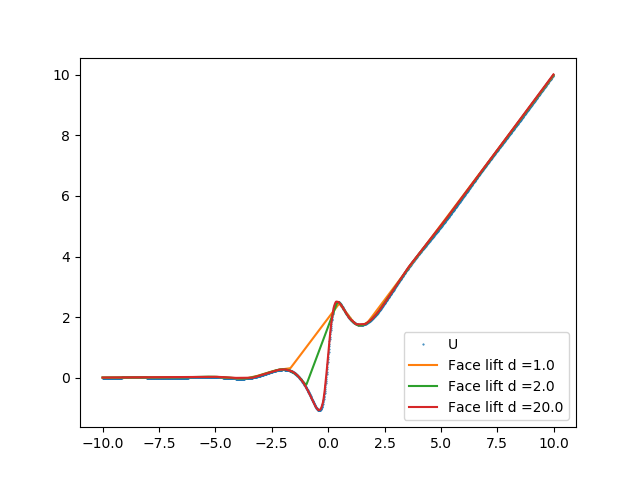}
 \caption*{$\epsilon_2 =\frac{1}{10} $}
 \end{minipage}
 \caption{facelift approximation for different values of $\hat d$ for the third case. \label{fig:facelift3}}
\end{figure}
As we can see on Figure \ref{fig:dfacelift3}, constraints are well respected for test case 3.
\begin{figure}[H]
\begin{minipage}[b]{0.49\linewidth}
  \centering
 \includegraphics[width=\textwidth]{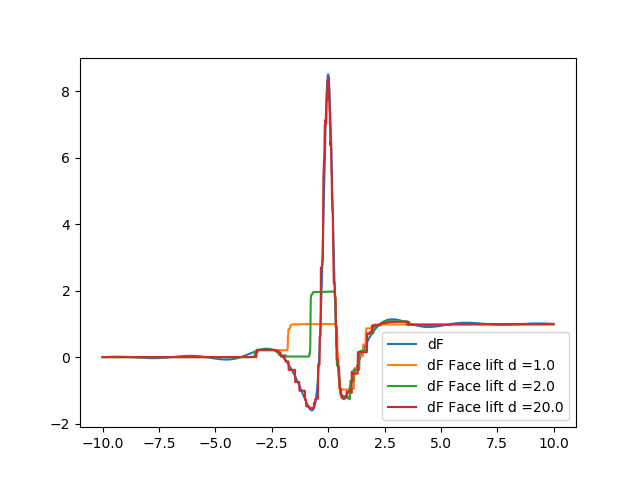}
 \caption*{$\epsilon_2 =\frac{1}{100}$}
 \end{minipage}
 \centering
 \begin{minipage}[b]{0.49\linewidth}
  \centering
 \includegraphics[width=\textwidth]{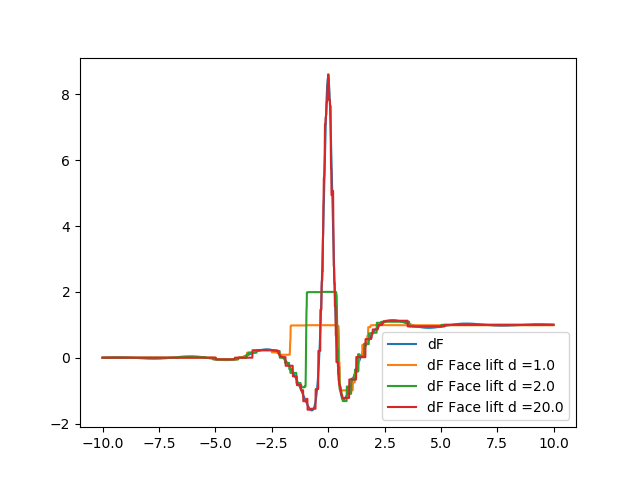}
 \caption*{$\epsilon_2 =\frac{1}{10} $}
 \end{minipage}
 \caption{Derivative of the facelift approximation for different values of $\epsilon_2$ for the third case \label{fig:dfacelift3}}
\end{figure}

\subsection{Results in higher dimension}
We extend the $\varphi$ function  given by  \eqref{eq:varphi2} in higher dimension by
\beq
\varphi(x)=   \frac{1}{d}\sum_{i=1}^d 4 [(x_i - 0.8)^{+} - (x_i-1)^{+}] +  (x_i-1.2)^{+}, \quad  x \in \R^d\;.
\label{eq:varphi2nD}
\enq
As before the facelift can be calculated analytically  for $d \hat d  \le 4$ as  $$\varphi^{A,\hat d}(x) = \frac{1}{d}\sum_{i=1}^d \varphi^A_{\hat d}(x_i)$$ where $\varphi^A_{\hat d}$ is given by equation \eqref{eq:varphiA1D}.\\
We test the accuracy of the facelift calculated $NN^{\theta_{m,\eps}^{*,k}}$
in different dimension by plotting
$$  \E[ (NN^{\theta_{m,\eps}^{*,k}} - \varphi^A)^2(\xi)]$$ with respect to $k$ for $\epsilon_2= \frac{1}{4000}$ for different values of $d$ and $\hat d$.
\begin{Remark}
Taking a very small value permits to get better results in high dimension but increases the number of iterations for easier cases.
\end{Remark}
On Figure \ref{fig:dfaceliftError}, we plot the error due to the algorithm with respect to number of  iterations for different  dimensions. Iterations are stopped below 10 when errors starts increasing meaning that the solution estimated is below the true one. In real application, a check on the $L_2$ difference between the estimation and the function to facelift is used to stop the iterations. 
\begin{figure}[H]
\begin{minipage}[b]{0.49\linewidth}
  \centering
 \includegraphics[width=\textwidth]{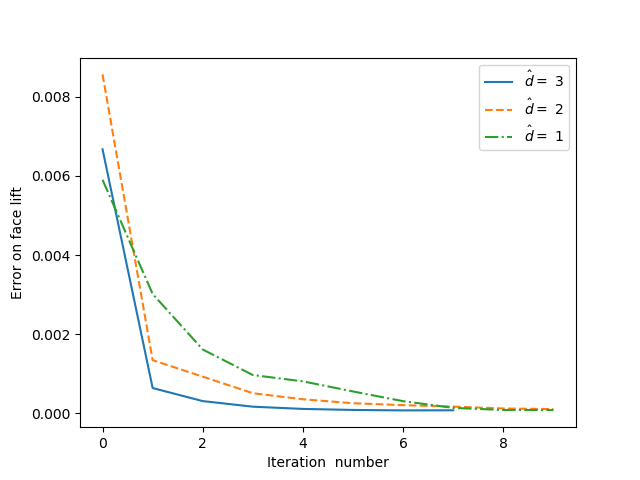}
 \caption*{$1D$}
 \end{minipage}
 \centering
 \begin{minipage}[b]{0.49\linewidth}
  \centering
 \includegraphics[width=\textwidth]{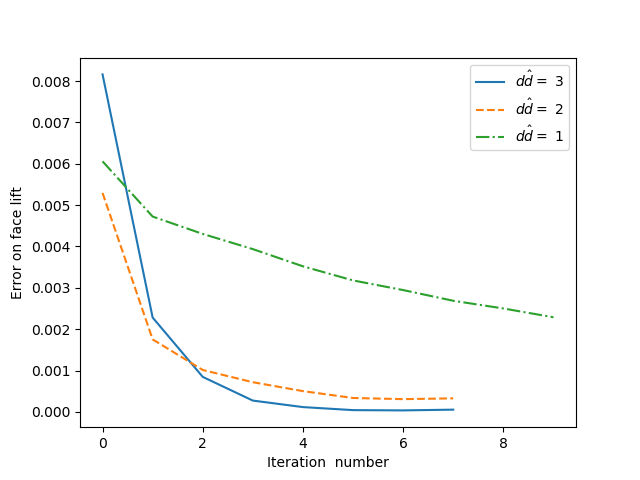}
 \caption*{$2D$}
 \end{minipage}
 \begin{minipage}[b]{0.49\linewidth}
  \centering
 \includegraphics[width=\textwidth]{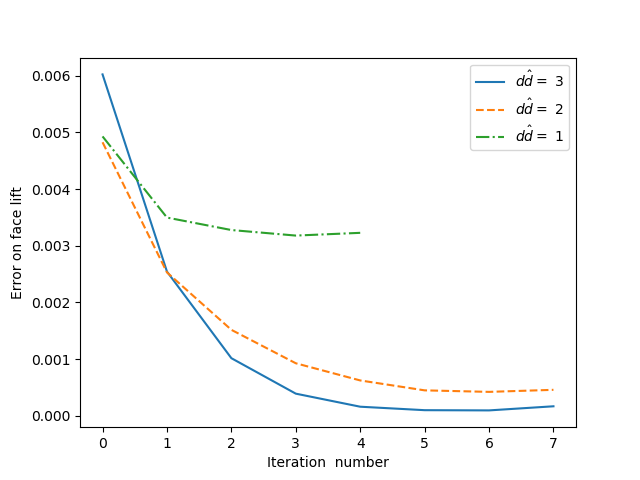}
 \caption*{$4D$}
 \end{minipage}
 \centering
 \begin{minipage}[b]{0.49\linewidth}
  \centering
 \includegraphics[width=\textwidth]{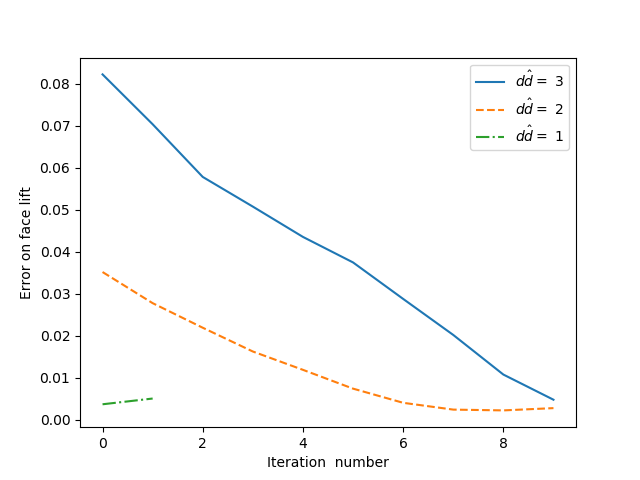}
 \caption*{$10D$}
 \end{minipage}
 \caption{Error with respect to the number of iterations \label{fig:dfaceliftError}}
\end{figure}
As expected, the convergence in dimension 10D is harder to achieve and hard constraints (small $\hat d$) are difficult to solve.
\subsubsection{Solving the BSDE with constraints}
\label{sec:BSDEConst}
In this  section we propose to solve the problem of option pricing with differential interest rates \cite{bergman1995option} adding a constraint on the number of shares held  in the portfolio. 
The  forward process is given by the Black Scholes model
\beqs
 X^{t,x}_s = x + \int_t^s \mu X_u^{t,x} du + \int_t^s \sigma X_u^{t,x}  dB_u \;,  \quad s \in [t,T] .
\enqs
The driver is given by
\beqs
f(t,x,y,z)= - r y -  z.( \sigma^{-1}(\mu-r) x) - (R-r)( y - z.(\sigma^{-1} x)).
\enqs
As the facelift is calculated by a neural network, it seems natural to solve the transition problem between two time steps by the same methodology.
Currently two effective methods have been developed to solve this problem \cite{HPW19} and \cite{beck2019deep}.
It turns out that the method given in \cite{HPW19} is more accurate than the method given in \cite{beck2019deep}. Then we apply the method given in \cite{HPW19} as described in the previous section to our problem.
We decide to apply the constraint after each resolution so we take $n_{k}=1$ in the implemented algorithm.\\
The parameters are taken as follows: we keep as for the facelift calculation two hidden layers with 200 neurons. For the activation function we keep the tanh function used in \cite{HPW19}.
The size of mini batch is taken equal to $1000$, and we check the convergence every $100$ epoch iterations. When reduction of the loss is not effective enough we reduce the learning rate with the methodology explained in \cite{chan2019machine}. Total number of iterations is limited to $50000$ for each time step.
Numerical test show that the number of neurons could be lower  and the activation function could be a ReLU or ELU : taking  50 and 100 neurons gives very similar results for activation functions listed above.\\
In one dimension, we give the results obtained for the second payoff function  used  in Section \ref{sec:flCase2}.
We take $T=1$, $r=0.05$, $\mu=0.07$, $\sigma =0.3$, $x=1$ for the initial asset value. The convex set is a ball of radius $\hat d$, $\epsilon_2$ taken equal to $\frac{1}{50}$ and the number of iterations $K$ in algorithm \ref{algo1} equal to $2$.
We give results obtained for different value of $R$. Taking $R=r$, we get a semi analytical value by taking the expectation of the facelift payoff of the process under the risk neutral measure as explained in \cite{BCS98}. This expectation is calculated by taking $1e7$ trajectories.\\
When $R \neq r$, no solution is available for this non linear problem.
In Tables \ref{tab:optionFL1:1},\ref{tab:optionFL1:2},  we give the results obtained with $20$ time steps for different values of $\hat d$. We give the average of $5$ calculations and the standard deviation of the results.
Notice that  without constraints and $R=r$, the analytical solution is $0.558$.
\begin{table}[H]
\begin{minipage}{.5\linewidth}
\caption*{$\hat d = 3$}
\centering
\begin{tabular}{|c|c|c|c|}  \hline
$ R$ &  0.05 &  0.07 & 0.09 \\ \hline
Analytical & 0.591   &      &     \\ \hline
Numerical &  0.598    & 0.612     & 0.627     \\ \hline
Std       &  0.004    & 0.002     &  0.0008     \\ \hline
\end{tabular}
 \end{minipage}%
 \begin{minipage}{.5\linewidth}
\centering
 \caption*{$\hat d = 2$}
\begin{tabular}{|c|c|c|c|}  \hline
 $R$ &  0.05 &  0.07 & 0.09 \\ \hline
 Analytical & 0.648   &      &     \\ \hline
 Numerical &  0.653   & 0.661     & 0.669     \\ \hline
 Std       &  0.001   &  0.002    &  0.003      \\ \hline
 \end{tabular}
\end{minipage} 
\caption{\label{tab:optionFL1:1} Results obtained by neural network algorithm  for payoff  \ref{eq:varphi2} using $20$ time steps ($n=21$)}
\end{table}
 \begin{table}[H]
 \centering
\begin{tabular}{|c|c|c|c|}  \hline
 $R$ &  0.05 &  0.07 & 0.09 \\ \hline
 Analytical & 0.736    &      &     \\ \hline
 Numerical &  0.742    & 0.743     &  0.744    \\ \hline
 Std       &  0.003    &  0.002    &  0.0009      \\ \hline
 \end{tabular}
 \caption{\label{tab:optionFL1:2} Results obtained by neural network algorithm  for payoff  \ref{eq:varphi2} using $20$ time steps ($n=21$),  $\hat d = 1.$}
 \end{table}
 \subsubsection{Results in higher dimension}
 In this section we take $R=r$ such that we get an analytical solution and we use the previous algorithm with the payoff  \eqref{eq:varphi2nD}.
\begin{table}[H]
\caption{\label{tab:optionFLnD1} Results obtained by neural network algorithm  for payoff  \ref{eq:varphi2nD} using $20$ time steps ($n=21$) with $\epsilon_2=  \frac{1}{50}$}
\centering
\begin{tabular}{|c|c|c|c|c|c|c|c|c|c|}  \hline
$d$  &  \multicolumn{3}{c|} {2}   & \multicolumn{3}{c|} {4}   & \multicolumn{3}{c|} {6}  \\ \hline
$ d \hat d$  & 3 & 2 & 1   &  3 & 2 & 1   & 3 & 2 & 1  \\ \hline
Ref & 0.591 & 0.648 & 0.736 & 0.591 & 0.648 & 0.736 & 0.591 & 0.648 & 0.736  \\ \hline
Num &  0.592  & 0.644  &  0.739   &  0.591    & 0.637 & 0.722 & 0.591    &  0.631     &  0.707  \\ \hline
Std &  0.002  &  0.002    &  0.001  &  0.002     &  0.003    &  0.003 & 0.001 &  0.002    &  0.002   \\ \hline
\end{tabular}
\end{table}
\begin{table}[H]
\caption{\label{tab:optionFLnD2} Results obtained by neural network algorithm  for payoff  \ref{eq:varphi2nD} using $20$ time steps ($n=21$) with $\epsilon_2=  \frac{1}{250}$}
\centering
\begin{tabular}{|c|c|c|c|c|c|c|c|c|c|}  \hline
$d$  &  \multicolumn{3}{c|} {2}   & \multicolumn{3}{c|} {4}   & \multicolumn{3}{c|} {6}  \\ \hline
$ d \hat d$  & 3 & 2 & 1   &  3 & 2 & 1   & 3 & 2 & 1  \\ \hline
Ref & 0.591 & 0.648 & 0.736 & 0.591 & 0.648 & 0.736 & 0.591 & 0.648 & 0.736  \\ \hline
Num &    0.598  & 0.6550  &  0.760 & 0.602    & 0.653     & 0.749 & 0.607    & 0.654    &   0.749  \\ \hline
Std       &   0.005   &  0.003    & 0.006  & 0.002  & 0.001  & 0.013    &  0.003      &  0.002    &   0.008  \\ \hline
\end{tabular}
\end{table}
As noticed before, since the constraints gets tighter, the results are not as good. Taking a very small $\eps$ gives  results with a higher standard deviation.
\appendix

\section{Regularity estimates on solutions to parabolic semi-linear PDEs}
We recall in this appendix an existence and uniqueness results for viscosity solution  to semi-linear PDEs.    We also give a regularity property with an explicit form for the Lipschitz and H\"older constants.  Although, this regularity is classical in PDE theory, we choose to provide such a result as we did not find any explicit mention of the dependence of the regularity coefficient in the literature.

We fix $\underline t,\overline t\in [0,T]$ and we consider a PDE of the form
\begin{equation}\label{PDE-appendix}\left\{
\begin{array}{l}
-\partial_t  w(t,x)-\Lc  w(t,x) \\
-h\big(t,x, w(t,x),\sigma(t,x)^\top D w(t,x)\big)  =  0\;,~~  (t,x)\in [\underline t,\overline t)\times\R^d\\
 w(\overline t,x)  =  m(x)\;,~~ x\in\R^d\\
\end{array}\right.
\end{equation}
We make the following assumption on the coefficients $m$ and $h$.

\vspace{2mm}

\ni  \textbf{(H$h,m$)}\begin{enumerate}[(i)]
\item The function $m$ is bounded: there exists a constant $M_m$ such that
\beqs
|m(x)| \leq M_m
\enqs
for all $x\in\R^d$.
\item The function $h$ is continuous and satisfies the following growth property: there exists a constant $M_h$ such that
\beqs
|h(t,x,y,z))| & \leq & M_h\big(1+|y|+|z|\big)
\enqs
for all $t\in[0,T]$, $x\in\R^d$, $y\in\R$ and $z\in\R^d$.
\item The functions $h$ and $m$ are Lipschitz continuous in their space variables uniformly in their time variable: there exist two constants $L_h$ and $L_m$ such that
\beqs
|m(x)-m(x')| & \leq &
 L_m|x-x'| \\
|h(t,x,y,z)-h(t,x',y',z')| & \leq &
 L_h\big(|x-x'|+|y-y'|+|z-z'|\big) 
\enqs
for all $t\in[0,T]$, $x,x'\in\R^d$,  $y,y'\in\R$ and $z,z'\in\R^d$.
\end{enumerate}

\vspace{2mm}

\begin{Proposition}\label{prop-reg-space-EDPSL}
Suppose  \textbf{(H$b,\sigma$)} and \textbf{(H$h,m$)} hold.  The PDE \reff{PDE-appendix} admits a unique viscosity solution $w$ with polynomial growth: there exist an integer $p\geq1$ and a constant $C$ such that
\beqs
|w(t,x)| & \leq & C(1+|x|^p)\;,\quad (t,x)\in[\underline t, \overline t]\times \R^d.
\enqs
Moreover, $w$ satisfies the following space regularity property
\beqs
|w(t,x)-w(t,x')| & \leq &  \\
e^{(2L_{b,\sigma}+L_{b,\sigma}^2+(L_h\vee 2)^2)(\overline t-\underline t)}(1+(\overline t - \underline t))^{1\over 2}\big( L_m^2+(\overline t - \underline t)(L_h\vee 2)^2\big)^{1\over 2}|x-x'| & &
\enqs
for all $t\in [\underline t, \overline t]$ and $x,x'\in\R^d$.
\end{Proposition}
We first need the following lemma.
\begin{Lemma}
 Under \textbf{(H$b,\sigma$)} we have the following estimate
\beq\label{estim-reg-diff2}
\sup_{s\in[t\vee t',\overline t]}\E\Big[|X^{t,x}_s-X^{t',x'}_s|^2\Big] & \leq & \\\nonumber
e^{(2L_{b,\sigma}+L_{b,\sigma}^2)(\overline t-\underline t)}(1+(\overline t - \underline t))\Big( |x-x'|+M_{b,\sigma}\sqrt{|t-t'|}\Big)^2 & &
\enq
for $t,t'\in[\underline t,\overline t]$ and $x,x'\in \R^d$.
 \end{Lemma}
 \begin{proof} Fix $t,t'\in[\underline t,\overline t]$  such that $t'\leq t$ and $x,x'\in \R^d$. From It\^o's formula and \textbf{(H$b,\sigma$)} we have
 \beqs
\E\Big[|X^{t,x}_s-X^{t',x'}_s|^2\Big]  & \leq & \E\Big[|X^{t,x}_t-X^{t',x'}_t|^2\Big] + (2L_{b,\sigma}+L_{b,\sigma}^2)\int_t^s\E\Big[|X^{t,x}_u-X^{t',x'}_u|^2\Big] du
 \enqs
 for $s\in [t,\overline t]$. By Gronwall's Lemma we get
 \beqs
 \sup_{s\in[t,\overline t]}\E\Big[|X^{t,x}_s-X^{t',x'}_s|^2\Big] & \leq & \E\Big[|X^{t,x}_t-X^{t',x'}_t|^2\Big]e^{(2L_{b,\sigma}+L_{b,\sigma}^2)(\overline t-\underline t)}\;.
 \enqs
 Moreover, we have 
\beqs
\E\Big[|X^{t,x}_t-X^{t',x'}_t|^2\Big] & = & \E\Big[|x-x'-\int_{t'}^tb(s,X^{t',x'}_s)ds-\int_{t'}^t\sigma(s,X^{t',x'}_s)dB_s|^2\Big]\\
 & \leq & |x-x'|^2+M^2_{b,\sigma}|t-t'|^2+M^2_{b,\sigma}|t-t'|+2 M_{b,\sigma}|x-x'||t-t'|\\
 & \leq & \Big( |x-x'|+M_{b,\sigma}\sqrt{|t-t'|}\Big)^2(1+(\overline t - \underline t)
 ))\;.
\enqs 
Which give the result. \end{proof}
 
\begin{proof}[Proof of Proposition \ref{prop-reg-space-EDPSL}]
For $(t,x)\in[\underline t, \overline t]\times \R^d$, we introduce the following BSDE: find $(\Yc^{t,x},\Zc^{t,x})\in \mathbf{S}^2_{[t,\overline t]}\times \mathbf{H}^2_{[t,\overline t]}$ such that
\beqs
\Yc^{t,x}_u & = & m(X^{t,x}_{\overline t})+\int_u^{\overline t}h(s,X^{t,x}_s, \Yc^{t,x}_s,\Zc^{t,x}_s)ds-\int_u^{\overline t}\Zc_s^{t,x} dB_s\;,~~u\in[t,\overline t]\;.
\enqs
From Theorem 1.1 in \cite{pardoux1998backward}, we get the existence and uniqueness of the solution to this BSDE for all $(t,x)\in[\underline t, \overline t]\times \R^d$.
From Theorem 2.2 in  \cite{pardoux1998backward} and Theorem 5.1 in \cite{PPR97}, the function $w$ defined by
\beqs
 w(t,x)  & = & \Yc^{t,x}_t \;,\quad (t,x)\in[\underline t, \overline t]\times \R^d\;,
\enqs
 is continuous and is the unique viscosity solution to \reff{PDE-appendix} with polynomial growth. 
 
 We now turn to the regularity estimate. 
 We first check the regularity w.r.t. the variable $x$.

 Fix $t\in[\underline t, \overline t]$ and $x,x' \in\R^d$. By It\^o's formula we have
 \beqs
 |\Yc^{t,x}_s-\Yc^{t,x'}_{s}|^2
  & = & |m(X^{t,x}_{\overline t})-m(X^{t,x'}_{\overline t})|^2\\
   & & +\int_{s}^{\overline t}\big( h(u,X^{t,x}_{u},\Yc^{t,x}_u,\Zc^{t,x}_u)-h(u,X^{t,x'}_{u},\Yc^{t,x'}_u,\Zc^{t,x'}_u) \big)(\Yc^{t,x}_u-\Yc^{t,x'}_{u})du\\
    &  & - \int_{s}^{\overline t}|\Zc^{t,x}_u-\Zc^{t,x'}_u|^2du-\int_{s}^{\overline t}(\Yc^{t,x}_u-\Yc^{t,x'}_u)(\Zc^{t,x}_u-\Zc^{t,x'}_u).dB_u
 \enqs 
for $s\in[t,\overline t]$. Using Lipschitz properties of $h$ and $m$ and Young ineqality we get
\beqs
 \E[|\Yc^{t,x}_s-\Yc^{t,x'}_{s}|^2] & \leq & L_m^2\E[|X^{t,x}_{\overline t}-X^{t,x'}_{\overline t}|^2]+L_h^2\int_{s}^{\overline t}\E[|X^{t,x}_{u}-X^{t,x'}_{u}|^2]du
\\
 & & + ({L_h^2\over 4}+L_h+1)\int_{s}^{\overline t} \E[|\Yc^{t,x}_u-\Yc^{t,x'}_{u}|^2] du\;,\quad s\in[t,\overline t]\;.
\enqs
Since $({x^2\over 4}+x+1)\leq x^2$ for $x\geq 2$ we get
\beqs
 \E[|\Yc^{t,x}_s-\Yc^{t,x'}_{s}|^2] & \leq & L_m^2\E[|X^{t,x}_{\overline t}-X^{t,x'}_{\overline t}|^2]+(L_h\vee 2)^2\int_{s}^{\overline t}\E[|X^{t,x}_{u}-X^{t,x'}_{u}|^2]du
\\
 & & + (L_h\vee 2)^2\int_{s}^{\overline t} \E[|\Yc^{t,x}_u-\Yc^{t,x'}_{u}|^2] du\;,\quad s\in[t,\overline t]\;.
\enqs
Then using \reff{estim-reg-diff2}, we get
\beqs
\E[|\Yc^{t,x}_s-\Yc^{t,x'}_{s}|^2]  & \leq &e^{(2L_{b,\sigma}+L_{b,\sigma}^2)(\overline t-\underline t)}(1+(\overline t - \underline t)) \big( L_m^2+(\overline t - \underline t)(L_h\vee 2)^2\big)|x-x'|^2 \\
 & & + (L_h\vee 2)^2\int_{s}^{\overline t} \E[|\Yc^{t,x}_u-\Yc^{t,x'}_{u}|^2] du\;,\quad s\in[t,\overline t]\;.
\enqs
Using Gronwall's Lemma we get
\beqs
\E[|\Yc^{t,x}_t-\Yc^{t,x'}_{t}|^2]  & \leq & e^{(2L_{b,\sigma}+L_{b,\sigma}^2+(L_h\vee 2)^2)(\overline t-\underline t)}(1+(\overline t - \underline t))\big( L_m^2+(\overline t - \underline t)(L_h\vee 2)^2\big) |x-x'|^2\;. 
\enqs
Therefore, we get
\beqs
|w(t,x)-w(t,x')| & \leq & e^{(2L_{b,\sigma}+L_{b,\sigma}^2+(L_h\vee 2)^2)(\overline t-\underline t)}(1+(\overline t - \underline t))^{1\over 2}\big( L_m^2+(\overline t - \underline t)(L_h\vee 2)^2\big)^{1\over 2} |x-x'|\;. 
\enqs
 \end{proof}
 In this last result we prove that under our assumptions the $Z$ component of a solution to a BSDE is bounded.
We recall that $(\Yc^{t,x},\Zc^{t,x})\in \mathbf{S}^2_{[t,\overline t]}\times \mathbf{H}^2_{[t,\overline t]}$ denotes the solution to 
\beqs
\Yc^{t,x}_u & = & m(X^{t,x}_{\overline t})+\int_u^{\overline t}h(s,X^{t,x}_s, \Yc^{t,x}_s,\Zc^{t,x}_s)ds-\int_u^{\overline t}\Zc_s^{t,x} dB_s\;,~~u\in[t,\overline t]\;,
\enqs
for $(t,x)\in [\underline t,\overline t]\times\R^d$.
\begin{Proposition}\label{propZborne}
Under \textbf{(H$b,\sigma$)} and \textbf{(H$h,m$)}, the process $\Zc^{t,x}$ satisfies
\beqs
|\Zc^{t,x}| & \leq & M_{b,\sigma}e^{(2L_{b,\sigma}+L_{b,\sigma}^2+(L_h\vee 2)^2)(\overline t-\underline t)}(1+(\overline t - \underline t))^{1\over 2}\big( L_m^2+(\overline t - \underline t)(L_h\vee 2)^2\big)^{1\over 2}  \enqs
$d\P\otimes dt ~a.e.$  on $\Omega\times  [t,\overline t]$.

\end{Proposition}
\begin{proof}
By a mollification argument, we can find regular functions $b_n$ and $\sigma_n$ satisfying \textbf{(H$b,\sigma$)}  with same constants as $b$, $\sigma$,  $h_n$ and $m_n$ satisfying \textbf{(H$h,m$)} with same constants as $h$ and $m$ for $n\geq 1$ such that
\beq\label{approx-reg-m-h}
(b_n,\sigma_n,h_n,m_n) & \xrightarrow[n\rightarrow+\infty]{} & (b,\sigma,h,m)\;,
\enq
uniformly on compact sets.
We fix now $(t,x)\in [\underline t,\overline t]\times\R^d$ and we denote by
$(X^{t,x,n},\Yc^{t,x,n},\Zc^{t,x,n})\in \mathbf{S}^2_{[t,\overline t]}\times\mathbf{S}^2_{[t,\overline t]}\times \mathbf{H}^2_{[t,\overline t]}$ the solution to 
\beqs
 X^{t,x,n}_u & = & x+\int_t^ub_n(s,X^{t,x,n}_s)ds+\int_t^u\sigma_n(s,X^{t,x,n}_s)dB_s\;,\quad u\in[t,\bar t],
\\
\Yc^{t,x,n}_u & = & m_n (X^{t,x,n}_{\overline t})+\int_u^{\overline t}h_n (s,X^{t,x,n}_s, \Yc^{t,x,n}_s,\Zc^{t,x,n}_s)ds-\int_u^{\overline t}\Zc_s^{t,x,n} dB_s\;,~~u\in[t,\overline t]\;.
\enqs
From \reff{approx-reg-m-h} we get
\beq\label{approxZconv}
{\|Y^{t,x}-Y^{t,x,n}\|}_{\mathbf{S}^2_{[t,\bar t]}}+{\|Z^{t,x}-Z^{t,x,n}\|}_{\mathbf{H}^2_{[t,\bar t]}} & \xrightarrow[n\rightarrow+\infty]{} & 0\;.
\enq
 From Theorem 3.2 in \cite{PP91}, we have 
\beqs
Y^{t,x,n}_s & = & w_n(s, X^{t,x,n})\quad s\in[t,\bar t]\;,
\enqs 
 where $w_n$ is a regular solution to 
 \begin{equation*}\left\{
\begin{array}{l}
-\partial_t  w_n-\Lc  w_n -h_n\big(., w_n,\sigma_n^\top D w_n\big)  =  0\;,~~ \mbox{ on } [\underline t,\overline t)\times\R^d\;,\\
 w_n(\overline t,.)  =  m_n\;,~~ \mbox{ on }\R^d\;.\\
\end{array}\right.
\end{equation*}
From the uniqueness of solutions to Lipschitz BSDEs   we get by applying It\^o's formula
 \beqs
 Z^{t,x,n}_s &= & (\sigma_n^\top Dw_n)(s,X^{t,x}_s)\;,\quad s\in[t,\overline t]\;.
 \enqs
Since $\sigma_n$, $m_n$ and $h_n$ satisfy \textbf{(H$h,m$)}, we get from Proposition \ref{prop-reg-space-EDPSL}
\beqs
\sup_{[\underline t, \overline t]\times\R^d}|Dw_n| & \leq & e^{(2L_{b,\sigma}+L_{b,\sigma}^2+(L_h\vee 2)^2)(\overline t-\underline t)}(1+(\overline t - \underline t))^{1\over 2}\big( L_m^2+(\overline t - \underline t)(L_h\vee 2)^2\big)^{1\over 2}\;.
\enqs
Therefore, we have
\beqs
|\Zc^{t,x,n}_s| & \leq & M_{b,\sigma}e^{(2L_{b,\sigma}+L_{b,\sigma}^2+(L_h\vee 2)^2)(\overline t-\underline t)}(1+(\overline t - \underline t))^{1\over 2}\big( L_m^2+(\overline t - \underline t)(L_h\vee 2)^2\big)^{1\over 2}  \enqs
$d\P\otimes dt ~a.e.$  on $\Omega\times  [t,\overline t]$.
We then conclude using \reff{approxZconv}. \end{proof}
\begin{Proposition}\label{prop-reg-time-EDPSL}
Under \textbf{(H$b,\sigma$)} and \textbf{(H$h,m$)} the unique viscosity solution with linear growth $w$  \reff{PDE-appendix} satisfies the following time regularity property
\beqs
|w(t,x)-w(t',x)| & \leq & \\
 e^{(3L_{b,\sigma}+2L_{b,\sigma}^2+(L_h\vee 2)^2)(\overline t-\underline t)}(1+(\overline t - \underline t))^{3\over 2}\big( L_m^2+(\overline t-\underline t)(L_h\vee 2)^2\big)^{1\over 2}M_{b,\sigma}\sqrt{t-t'} & &\\
 +  M_h\big(M_m+ M_h(\overline t-\underline t)\big)e^{M_h(\overline t-\underline t)}(t'-t)  
 & &\\
 + M_hM_{b,\sigma}e^{(2L_{b,\sigma}+L_{b,\sigma}^2+(L_h\vee 2)^2)(\overline t-\underline t)}(1+(\overline t - \underline t))^{1\over 2}\big( L_m^2+(\overline t - \underline t)(L_h\vee 2)^2\big)^{1\over 2}  (t'-t) & &
\enqs
for all $t,t'\in [\underline t, \overline t]$ and $x\in\R^d$.
\end{Proposition}
\begin{proof}
We take the same notations as in the proof of Proposition \ref{prop-reg-space-EDPSL}. We fix $t,t'\in[\underline t , \overline t]$ such that $t'\leq t$ and $x\in \R^d$. We have
 \beqs
  |w(t,x)-w(t',x)| & = & |\Yc^{t,x}_t-\Yc^{t',x}_{t'}|\\
   & = & \Big|\Yc^{t,x}_t-\E\Big[\Yc^{t,X^{t',x}_t}_{t}+\int_{t'}^tf(s,X^{t',x}_s,\Yc^{t',x}_s,\Zc^{t',x}_s)ds\Big]\Big|\\
   & \leq & \E\Big[|\Yc^{t,x}_t-\Yc^{t,X^{t',x}_t}_{t}|  \Big] + M_h\int_{t'}^{t}(1+\E[|\Yc^{t',x}_s|]+\E[|\Zc^{t',x}_s|])ds
   \enqs
By a classical argument using \textbf{(H$h,m$)}, Young's inequality and  Grownwall's Lemma we have
\beqs
\sup_{s\in [\underline t, \overline t]}\E[|\Yc^{t',x}_s|^2] & \leq & M_m^2+ e^{4(M_h+M_h^2)(\overline t-\underline t)}\;.
\enqs
Then using Proposition \ref{propZborne} we have
\beqs
\E[|\Zc^{t',x}_s|] & \leq & M_{b,\sigma}e^{(2L_{b,\sigma}+L_{b,\sigma}^2+(L_h\vee 2)^2)(\overline t-\underline t)}(1+(\overline t - \underline t))^{1\over 2}\big( L_m^2+(\overline t - \underline t)(L_h\vee 2)^2\big)^{1\over 2}  
\enqs 
for $s\in [t',t]$.
From the regularity w.r.t. the variable $x$ given in Proposition \ref{prop-reg-space-EDPSL}
we get
\beqs
 |w(t,x)-w(t',x)| & \leq &\\
  e^{(2L_{b,\sigma}+L_{b,\sigma}^2+(L_h\vee 2)^2)(\overline t-\underline t)}(1+(\overline t - \underline t))^{1\over 2}\big( L_m^2+(\overline t-\underline t)(L_h\vee 2)^2\big)^{1\over 2} \E[|X^{t,x}_t-X^{t',x}_t|]&  & \\ 
 +  M_h\Big(M_m^2+ e^{4(M_h+M_h^2)(\overline t-\underline t)}+1\Big)(t'-t)   & &\\
 + M_hM_{b,\sigma}e^{(2L_{b,\sigma}+L_{b,\sigma}^2+(L_h\vee 2)^2)(\overline t-\underline t)}(1+(\overline t - \underline t))^{1\over 2}\big( L_m^2+(\overline t - \underline t)(L_h\vee 2)^2\big)^{1\over 2}  (t'-t)\;.  & &
 \enqs
From \reff{estim-reg-diff2} we get
\beqs
 |w(t,x)-w(t',x)| & \leq & \\
  e^{(3L_{b,\sigma}+2L_{b,\sigma}^2+(L_h\vee 2)^2)(\overline t-\underline t)}(1+(\overline t - \underline t))\big( L_m^2+(\overline t-\underline t)(L_h\vee 2)^2\big)^{1\over 2}M_{b,\sigma}\sqrt{t-t'}   & & \\
+  M_h\big(M_m^2+ e^{4(M_h+M_h^2)(\overline t-\underline t)}+1\big)(t'-t) & &\\
 + M_hM_{b,\sigma}e^{(2L_{b,\sigma}+L_{b,\sigma}^2+(L_h\vee 2)^2)(\overline t-\underline t)}(1+(\overline t - \underline t))^{1\over 2}\big( L_m^2+(\overline t - \underline t)(L_h\vee 2)^2\big)^{1\over 2}  (t'-t)\;.   & & 
\enqs

\end{proof}
 \bibliographystyle{plain}
\bibliography{BibKLW20}

\begin{thebibliography}{10}

\bibitem{Bar94}
Guy Barles.
\newblock {\em Solution de viscosit\'es des \'equations d'{H}amilton {J}acobi},
  volume~17 of {\em Math\'ematiques et Applications}.
\newblock Springer Verlag, 1994.

\bibitem{beck2019deep}
Christian Beck, Sebastian Becker, Patrick Cheridito, Arnulf Jentzen, and Ariel
  Neufeld.
\newblock Deep splitting method for parabolic pdes.
\newblock {\em arXiv preprint arXiv:1907.03452}, 2019.

\bibitem{bergman1995option}
Yaacov~Z Bergman.
\newblock Option pricing with differential interest rates.
\newblock {\em The Review of Financial Studies}, 8(2):475--500, 1995.

\bibitem{BCS98}
Mark Broadie, Jaksa Cvitani\'c, and Halil~Mete Soner.
\newblock Optimal replication of contingent claims under portfolioconstraints.
\newblock {\em The Review of Financial Studies}, 11:59--79, 1998.

\bibitem{BEM18}
Bouchard Bruno, Romuald Elie, and Ludovic Moreau.
\newblock Regularity of bsdes with a convex constraint on the gains-process.
\newblock {\em Bernoulli}, 24(3):1613--1635, 2018.

\bibitem{chan2019machine}
Quentin Chan-Wai-Nam, Joseph Mikael, and Xavier Warin.
\newblock Machine learning for semi linear pdes.
\newblock {\em Journal of Scientific Computing}, 79(3):1667--1712, 2019.

\bibitem{CEK20}
Jean-Fran\c{c}ois Chassagneux, Romuald Elie, and Kharroubi Idris.
\newblock A numerical probabilistic scheme for super-replication with convex
  constraints on the delta.
\newblock {\em Forthcoming}, 2020.

\bibitem{CrandallIshiiLions}
Michael~G. Crandall, Hitoshi Ishii, and Pierre-Louis Lions.
\newblock User's guide to viscosity solutions of second order partial
  differential equations.
\newblock {\em Bull. Amer. Math. Soc. (N.S.)}, 27(1):1--67, 1992.

\bibitem{CKS98}
Jak$\check{{\rm s}}$a Cvitani\'c, Ioannis Karatzas, and H.~Mete Soner.
\newblock Backward stochastic differential equations with constraints on the
  gains-process.
\newblock {\em The Annals of Probability}, 26(4):1522--1551, 1998.

\bibitem{DM78}
Claude Dellacherie and Paul-Andr\'e Meyer.
\newblock Probability and potential.
\newblock {\em Paris: Hermann}, 1978.

\bibitem{EKPQ97}
Nicole El~Karoui, Shige Peng, and Marie-Claire Quenez.
\newblock Backward stochastic differential equations in finance.
\newblock {\em Mathematical Finance}, 7(1):1--71, 1997.

\bibitem{gobet2008numerical}
Emmanuel Gobet and Jean-Philippe Lemor.
\newblock Numerical simulation of bsdes using empirical regression methods:
  theory and practice.
\newblock {\em arXiv preprint arXiv:0806.4447}, 2008.

\bibitem{HSW90}
Kurt Hornik, Maxewell Stinchcombe, and Halbert White.
\newblock Universal approximation of an unknown mapping and its derivatives
  using multilayer feedforward networks.
\newblock {\em Neural Networks}, 3(5):551--560, 1990.

\bibitem{HSW89}
Kurt Hornik, Maxwell Stinchcombe, and Halbert White.
\newblock Multilayer feedforward networks are universal approximators.
\newblock {\em Neural Networks}, 2(5):359--366, 1989.

\bibitem{HPW19}
C\^ome Hur\'e, Huy\^en Pham, and Xavier Warin.
\newblock Some machine learning schemes for high-dimensional nonlinear pdes.
\newblock {\em Mathematics of Computation}, Forthcoming.

\bibitem{kingma2014adam}
Diederik~P Kingma and Jimmy Ba.
\newblock Adam: A method for stochastic optimization.
\newblock {\em arXiv preprint arXiv:1412.6980}, 2014.

\bibitem{pardoux1998backward}
{\'E}tienne Pardoux.
\newblock Backward stochastic differential equations and viscosity solutions of
  systems of semilinear parabolic and elliptic pdes of second order.
\newblock In {\em Stochastic Analysis and Related Topics VI}, pages 79--127.
  Springer, 1998.

\bibitem{PP91}
Etienne Pardoux and Shige Peng.
\newblock Backward stochastic differential equations and quasilinear parabolic
  partial differential equations.
\newblock In Boris~L. Rozovskii and Richard~B. Sowers, editors, {\em Stochastic
  Partial Differential Equations and Their Applications}, volume 176 of {\em
  Lecture Notes in Control and Information Sciences}, pages 200--217. Springer,
  Berlin, Heidelberg, 1991.

\bibitem{PPR97}
Etienne Pardoux, Fr\'ed\'eric Pradeillles, and Zusheng Rao.
\newblock Probabilistic interpretation of a system of semi-linear parabolic
  partial differential equations.
\newblock {\em Annales de l'I. H. P., section B}, 33(4):467--490, 1997.

\bibitem{Peng99}
Shige Peng.
\newblock Monotonic limit theorem of {BSDE} and nonlinear decomposition theorem
  of {D}oob-{M}eyers type.
\newblock {\em Probability Theory and Related Fields}, 113(4):473--499, 1999.

\bibitem{Pham-book}
Huy\^en Pham.
\newblock {\em {C}ontinuous-time {S}tochastic {C}ontrol and {O}ptimization with
  {F}inancial {A}pplications}, volume~61 of {\em Stochastic modelling and
  applied probability}.
\newblock Springer-Verlag Berlin Heidelberg, 2009.

\bibitem{Rock70}
R.~Tyrrell Rockafellar.
\newblock {\em {C}onvex {A}nalysis}.
\newblock {P}rinceton {M}athematical {S}eries. Princeton University Press,
  1970.

\end{thebibliography}
\end{document}